\documentclass[nohyperref]{article}

\usepackage{microtype}
\usepackage{graphicx}
\usepackage{subfigure}
\usepackage{booktabs} 

\usepackage{amsfonts}       
\usepackage{nicefrac}       
\usepackage{xcolor}         

\usepackage{diagbox}
\usepackage{multirow}
\usepackage{wrapfig}
\usepackage{caption}
\usepackage[ruled, vlined]{algorithm2e}
\usepackage{enumitem}
\setlist{topsep=0pt, leftmargin=*}

\usepackage{amsmath,mleftright,mathtools,amsthm}
\DeclareMathAlphabet\mathcal{OMS}{cmsy}{m}{n}
\DeclareMathAlphabet\mathbfcal{OMS}{cmsy}{b}{n}
\usepackage{amssymb}

\usepackage{hyperref}

\usepackage[accepted]{icml2023}

\usepackage[capitalize,noabbrev]{cleveref}

\theoremstyle{plain}
\newtheorem{theorem}{Theorem}[section]

\theoremstyle{definition}
\newtheorem{definition}[theorem]{Definition}

\theoremstyle{remark}

\usepackage{Definitions}

\newif\ifcomment
\providecommand{\OM}[1]{\ifcomment{\small \color{blue} [OM: #1]}\fi} 
\providecommand{\R}[1]{\ifcomment{\small \color{purple} [R: #1]}\fi} 

\icmltitlerunning{Superhuman Fairness} 

\begin{document}

\twocolumn[
\icmltitle{Superhuman Fairness}




\begin{icmlauthorlist}
\icmlauthor{Omid Memarrast}{to}
\icmlauthor{Linh Vu}{to}
\icmlauthor{Brian Ziebart}{to}
\end{icmlauthorlist}

\icmlaffiliation{to}{Computer Science Department, University of Illinois Chicago}

\icmlcorrespondingauthor{O. Memarrast}{omemar2@uic.edu}

\icmlkeywords{Fairness, Classification, Imitation Learning, Machine Learning}

\vskip 0.3in
]



\printAffiliationsAndNotice{\icmlEqualContribution} 

\begin{abstract}
The fairness of machine learning-based decisions has become an increasingly important focus in the design of supervised machine learning methods.
Most fairness approaches optimize a specified trade-off between 
performance measure(s) (e.g., accuracy, log loss, or AUC) and fairness metric(s) (e.g., demographic parity, equalized odds). 
This begs the question: are the right performance-fairness trade-offs being specified?
We instead re-cast fair machine learning as an imitation learning task by introducing \emph{superhuman fairness}, which seeks to simultaneously outperform human decisions on multiple predictive performance and fairness measures.
We demonstrate the benefits of this approach given suboptimal decisions.

%
%
\end{abstract}

\section{Introduction}
\label{intro}
The social impacts of algorithmic decisions based on machine learning have motivated various group and individual fairness properties that decisions should ideally satisfy \cite{calders2009building,hardt2016equality}.
Unfortunately, impossibility results prevent multiple common group fairness properties from being simultaneously satisfied \cite{kleinberg2016inherent}.
Thus, no set of decisions can be universally fair to all groups and individuals for all notions of fairness. Instead, specified weightings, or trade-offs, of different criteria are often optimized \cite{liu2022accuracy}.
Identifying an appropriate trade-off to prescribe to these fairness methods is a daunting task open to application-specific philosophical and ideological debate that could delay or completely derail the adoption of algorithmic methods.

\begin{figure}[t]
\begin{center}
\begin{minipage}{0.6\columnwidth}
\includegraphics[width=0.95\columnwidth]{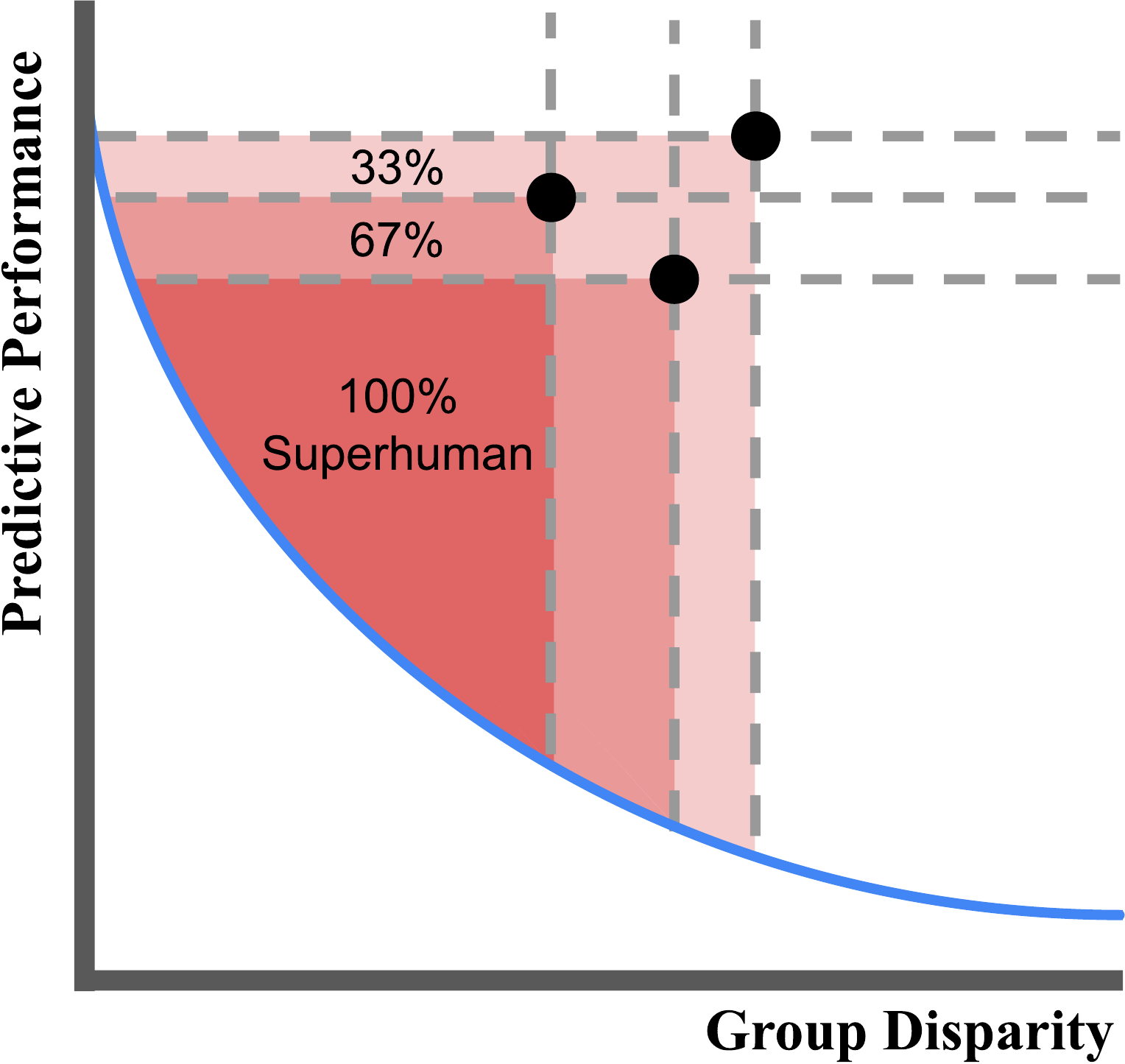} 
\end{minipage}
\begin{minipage}{0.38\columnwidth}
\vspace{-5mm}
\captionof{figure}{Three sets of decisions (black dots) with different predictive performance and group disparity values defining the sets of 100\%-, \mbox{67\%-,} and 33\%-superhuman fairness-performance values (red shades) based on Pareto dominance.}
\label{fig:superhuman}
\end{minipage}
\vspace{-7mm}
\end{center}
\end{figure}

We consider the motivating scenario of a fairness-aware decision task currently being performed by well-intentioned, but inherently error-prone human decision makers.
Rather than seeking optimal decisions for specific performance-fairness trade-offs, which may be difficult to accurately elicit, we propose a more modest, yet more practical objective: {\bf outperform human decisions across all performance and fairness measures with maximal frequency}.
We implicitly assume that available human decisions reflect desired performance-fairness trade-offs, but are often noisy and suboptimal.
This provides an opportunity for {\bf superhuman decisions} that Pareto dominate human decisions across predictive performance and fairness metrics (Figure \ref{fig:superhuman}) \emph{without identifying an explicit desired trade-off}.

To the best of our knowledge, this paper is the first to define fairness objectives for supervised machine learning with respect to noisy human decisions rather than using prescriptive trade-offs or hard constraints.
We leverage and extend a recently-developed imitation learning method for {\bf subdominance minimization} \cite{ziebart2022towards}. 
Instead of using the subdominance to identify a target trade-off, as previous work does in the inverse optimal control setting to estimate a cost function, we use it to directly optimize our fairness-aware classifier.
We develop policy gradient optimization methods \cite{sutton2018reinforcement} that allow flexible classes of probabilistic decision policies to be optimized for given sets of performance/fairness measures and demonstrations.



We conduct extensive experiments on standard fairness datasets ({\tt Adult} and {\tt COMPAS}) using accuracy as a performance measure and three conflicting fairness definitions: Demographic Parity \cite{calders2009building}, Equalized Odds \cite{hardt2016equality}, and Predictive Rate Parity \cite{chouldechova2017fair}).
Though our motivation is to outperform human decisions, we employ a synthetic decision-maker with differing amounts of label and group  membership noise to identify sufficient conditions for superhuman fairness of varying degrees.
We find that our approach achieves high levels of superhuman performance that increase rapidly with reference decision noise and significantly outperform the superhumanness of other methods that are based on more narrow fairness-performance objectives.


\R{The amount of contribution compared to previous work is not clear. In the abstract, the authors indicate that the approach is an extension of Ziebart et al. 2022 to the fair binary classification setting. However, from the section subdominance Minimization using Inverse Optimal Control, it seems that it is an application of the work for fair classification, not an extension?}

\OM{Discuss the contributions of the paper: 1) extension of Ziebart et al. not a trivial extensio 2) a general framework that works for non-binary classification and non-binary protected attributes. 3) it accept a general function as the classifier, we have developed logistic-regression and neural networks (hopefully). The NN version does not need protected attribute as an input since it can predict it. 4) The approach works for any subset of fairness/performance measures and }

\section{Fairness, Elicitation, and Imitation} 


\subsection{Group Fairness Measures}

Group fairness measures are primarily defined by confusion matrix statistics (based on labels $y_i \in \{0, 1\}$ and decisions/predictions $\hat{y}_i \in \{0, 1\}$ produced from inputs ${\bf x}_i \in \mathbb{R}^M$) for examples belonging to different protected groups (e.g., $a_i \in \{0, 1\}$).   

We 
focus on three prevalent fairness properties in this paper: 
\begin{itemize}
\item {\bf Demographic Parity} (DP) \cite{calders2009building} 
requires equal positive rates across protected groups:
\[
\mathrm{P}(\Yhat = 1 | A = 1) = \mathrm{P}(\Yhat = 1| A = 0);
\]
\item {\bf Equalized Odds} (EqOdds) \cite{hardt2016equality}
requires equal true positive rates and false positive rates across groups, i.e.,
{\small
\[
\mathrm{P}(\Yhat\! =\! 1 | Y\! =\! y, A\! =\! 1) = \mathrm{P}(\Yhat\! =\! 1| Y\! =\! y,  A\! =\! 0), \;\; y \in \{0,1\};
\]}
\item {\bf Predictive Rate Parity} (PRP) \cite{chouldechova2017fair}
requires equal positive predictive value ($\yhat=1$) and negative predictive value ($\yhat=0$) across groups:
{\small
\[
\mathrm{P}(Y\! =\! 1 | A\! =\! 1, \Yhat\! =\! \yhat) = \mathrm{P}(Y\! =\! 1| A\! =\! 0, \Yhat\! =\! \yhat), \;\; \yhat \in \{0,1\}.
\]}%
\end{itemize}
Violations of these fairness properties can be
 measured as differences:
{\small
\begin{align}
\texttt{D.DP}(\hat{\bf y},{\bf a}) = &\Bigg|
\frac{\sum_{i=1}^N \mathbb{I}\left[\hat{y}_i\!=\!1, a_i\!=\!1 \right]}{\sum_{i=1}^N \mathbb{I}\left[a_i\!=\!1\right]}
\\ &\qquad
-\frac{\sum_{i=1}^N \mathbb{I}\left[\hat{y}_i\!=\!1, a_i\!=\!0 \right]}{\sum_{i=1}^N \mathbb{I}\left[a_i\!=\!0\right]}
\Bigg|;\notag \\
 \texttt{D.EqOdds}(\hat{\bf y},{\bf y},{\bf a}) = & \!\!
\max_{y \in \{0,1\}} \!
\Bigg|
\frac{\sum_{i=1}^N \mathbb{I}\left[\hat{y}_i\!=\!1,y_i\!=\!y,a_i\!=\!1 \right]}{\sum_{i=1}^N \mathbb{I}\left[a_i\!=\!1,y_i\!=\!y\right]} \notag\\
&\!\!\!
-\frac{\sum_{i=1}^N \mathbb{I}\left[\hat{y}_i\!=\!1,y_i\!=\!y,a_i\!=\!0 \right]}{\sum_{i=1}^N \mathbb{I}\left[a_i\!=\!0,y_i\!=\!y\right]}\Bigg|; \\
  \texttt{D.PRP}(\hat{\bf y},{\bf y},{\bf a})
= &\!\!
\max_{y \in \{0,1\}}
\Bigg|
\frac{\sum_{i=1}^N \mathbb{I}\left[y_i\!=\!1,\hat{y}_i\!=\!y,a_i\!=\!1 \right]}{\sum_{i=1}^N \mathbb{I}\left[a_i\!=\!1,\hat{y}_i\!=\!y\right]} \notag\\
&\!\!\!
-\frac{\sum_{i=1}^N \mathbb{I}\left[y_i\!=\!1,\hat{y}_i\!=\!y,a_i\!=\!0 \right]}{\sum_{i=1}^N \mathbb{I}\left[a_i\!=\!0,\hat{y}_i\!=\!y\right]}\Bigg|.
\end{align}}%

\subsection{Performance-Fairness Trade-offs}

Numerous fair classification algorithms have been developed over the past few years, with most targeting one fairness metric \cite{hardt2016equality}\OM{add more citations?}.
With some exceptions \cite{blum2019recovering}, predictive performance and fairness are typically competing objectives in supervised machine learning approaches.
Thus, though satisfying many fairness properties simultaneously may be na\"ively appealing, doing so often significantly degrades predictive performance or even creates infeasibility \cite{kleinberg2016inherent}.

Given this, many approaches seek to choose parameters $\vec{\theta}$ for (probabilistic) classifier $\mathbb{P}_\vec{\theta}$ that balance the competing predictive performance and fairness objectives \cite{kamishima2012fairness,hardt2016equality,menon2018cost,celis2019classification,martinez2020minimax,rezaei2020fairness}.
Recently, \citet{hsu2022pushing} proposed a novel optimization framework to satisfy three conflicting fairness metrics (demographic parity, equalized odds, and predictive rate parity) to the best extent possible:
{\small
\begin{align}
& \min_\theta \mathbb{E}_{\hat{\bf y} \sim P_\theta}\Big[\textrm{loss}(\hat{\bf y},{\bf y}) + 
\alpha_{\texttt{DP}} \texttt{D.DP}(\hat{\bf y}, {\bf a}) \label{eq:joint_opt} \\
& \qquad + 
\alpha_{\texttt{Odds}} \texttt{D.EqOdds}(\hat{\bf y},{\bf y},{\bf a})  + 
\alpha_{\texttt{PRP}} \texttt{D.PRP}(\hat{\bf y},{\bf y},{\bf a})\Big]. \notag
\end{align}}
\subsection{Preference Elictation \& Imitation Learning}
Preference elicitation \cite{chen2004survey} is one natural approach to identifying desirable performance-fairness trade-offs. 
Preference elicitation methods typically query users for their pairwise preference on a sequence of pairs of options to identify their utilities for different characteristics of the options.
This approach has been extended and applied to fairness metric elicitation \cite{hiranandani2020fair}, allowing efficient learning of linear (e.g., Eq. \eqref{eq:joint_opt}) and non-linear metrics from finite and noisy preference feedback.

Imitation learning \cite{osa2018algorithmic} is a type of supervised machine learning that seeks to produce a general-use policy $\hat{\pi}$ based on demonstrated trajectories of states and actions, $\tilde{\xi}=\left(\tilde{s}_1, \tilde{a}_1, \tilde{s_2}, \ldots, \tilde{s}_T\right)$.  
Inverse reinforcement learning methods \cite{abbeel2004apprenticeship,ziebart2008maximum}
seek to rationalize the demonstrated trajectories as the result of (near-) optimal policies on an estimated cost or reward function.
Feature matching \cite{abbeel2004apprenticeship} plays a key role in these methods, guaranteeing if the expected feature counts match, the estimated policy $\hat{\pi}$ will have an expected cost under the demonstrator's unknown fixed cost function weights $\tilde{w} \in \mathbb{R}^K$ equal to the average of the demonstrated trajectories:
\begin{align}
&\mathbb{E}_{\xi \sim \hat{\pi}}\left[f_k(\xi)\right]=\frac{1}{N} \sum_{i=1}^N f_k\left(\tilde{\xi}_i\right), \forall k \label{eq:feature-matching}\\
&\Longrightarrow \mathbb{E}_{\xi \sim \hat{\pi}}\left[\operatorname{cost}_{\tilde{w}}(\xi)\right]=\frac{1}{N} \sum_{i=1}^N \operatorname{cost}_{\tilde{w}}\left(\tilde{\xi}_i\right),\notag
\end{align}
where $f_k(\xi)=\sum_{s_t \in \xi} f_k\left(s_t\right)$.

\citet{syed2007game} seeks to outperform the set of demonstrations when the signs of the unknown cost function are known, $\tilde{w}_k \geq 0$, by making the inequality,
\begin{align}
\mathbb{E}_{\xi \sim \pi}\left[f_k(\xi)\right]\leq \frac{1}{N} \sum_{i=1}^N f_k\left(\tilde{\xi}_i\right), \forall k,
\end{align}
strict for at least one feature.
Subdominance minimization \cite{ziebart2022towards} seeks to produce trajectories that outperform each demonstration by a margin:
\begin{align}
f_k(\xi) + m_k \leq f_k(\tilde{\xi}_i), \forall i, k,
\end{align}
under the same assumption of known cost weight signs.
However, since this is often infeasible, the approach instead minimizes the subdominance, which measures the $\alpha$-weighted violation of this inequality:
\begin{align}
&\operatorname{subdom}_{\alpha}(\xi, \tilde{\xi}) \triangleq\sum_k \left[\alpha_k\left(f_k(\xi)-f_k(\tilde{\xi})\right)+1\right]_{+},
\end{align}
where $[f(x)]_+ \triangleq \max(f(x), 0)$ is the hinge function and the per-feature margin has been reparameterized as $\alpha_k^{-1}$.
Previous work \cite{ziebart2022towards} has employed subdominance minimization in conjunction with inverse optimal control:
\begin{align}
& \min_{\bf w} \min_\alpha \sum_{i=1}^N \sum_{k=1}^K \operatorname{subdom}_{\alpha}(\xi^*({\bf w}), \tilde{\xi}_i), \text{where: } \notag\\ 
&\qquad\qquad\xi^*({\bf w}) = \operatornamewithlimits{\argmin}_{\xi}\sum_k w_k f_k(\xi), \notag
\end{align}
learning the cost function parameters ${\bf w}$ for the optimal trajectory $\xi^*({\bf w})$ that minimizes subdominance.
One contribution of this paper is extending subdominance minimization to the more flexible prediction models needed for fairness-aware classification that are not directly conditioned on cost features or performance/fairness metrics.

\section{Subdominance Minimization for Improved Fairness-Aware Classification}

We approach fair classification from an imitation learning perspective.
We assume vectors of (human-provided) reference decisions are available that roughly reflect desired fairness-performance trade-offs, but are also noisy.
Our goal is to construct a fairness-aware classifier that outperforms reference decisions on all performance and fairness measures on withheld data as frequently as possible.

\subsection{Superhumanness and Subdominance}

We consider reference decisions $\tilde{\yvec}=\{\tilde{y}_j\}_{j=1}^{\text{M}}$ that are drawn from a human decision-maker or baseline method $\tilde{\Pbb}$,
on a set of M items, $\Xvec_{\text{M}\times\text{L}}=\{\xvec_j\}_{j=1}^{\text{M}}$, where L is the number of attributes in each of $M$ items $\vec{x}_j$.
Group membership attributes $a_m$ from vector $\vec{a}$ indicate to which group item $m$ belongs.

The predictive performance and fairness of decisions $\hat{\yvec}$ for each item are assessed based on ground truth $\yvec$ and group membership $\bf{a}$ using a set of predictive loss and unfairness measures $\{f_k(\hat{\yvec},\yvec,\vec{a})\}$.

\begin{figure}[t]
\begin{center}
\begin{minipage}{0.6\columnwidth}
\includegraphics[width=0.95\columnwidth]{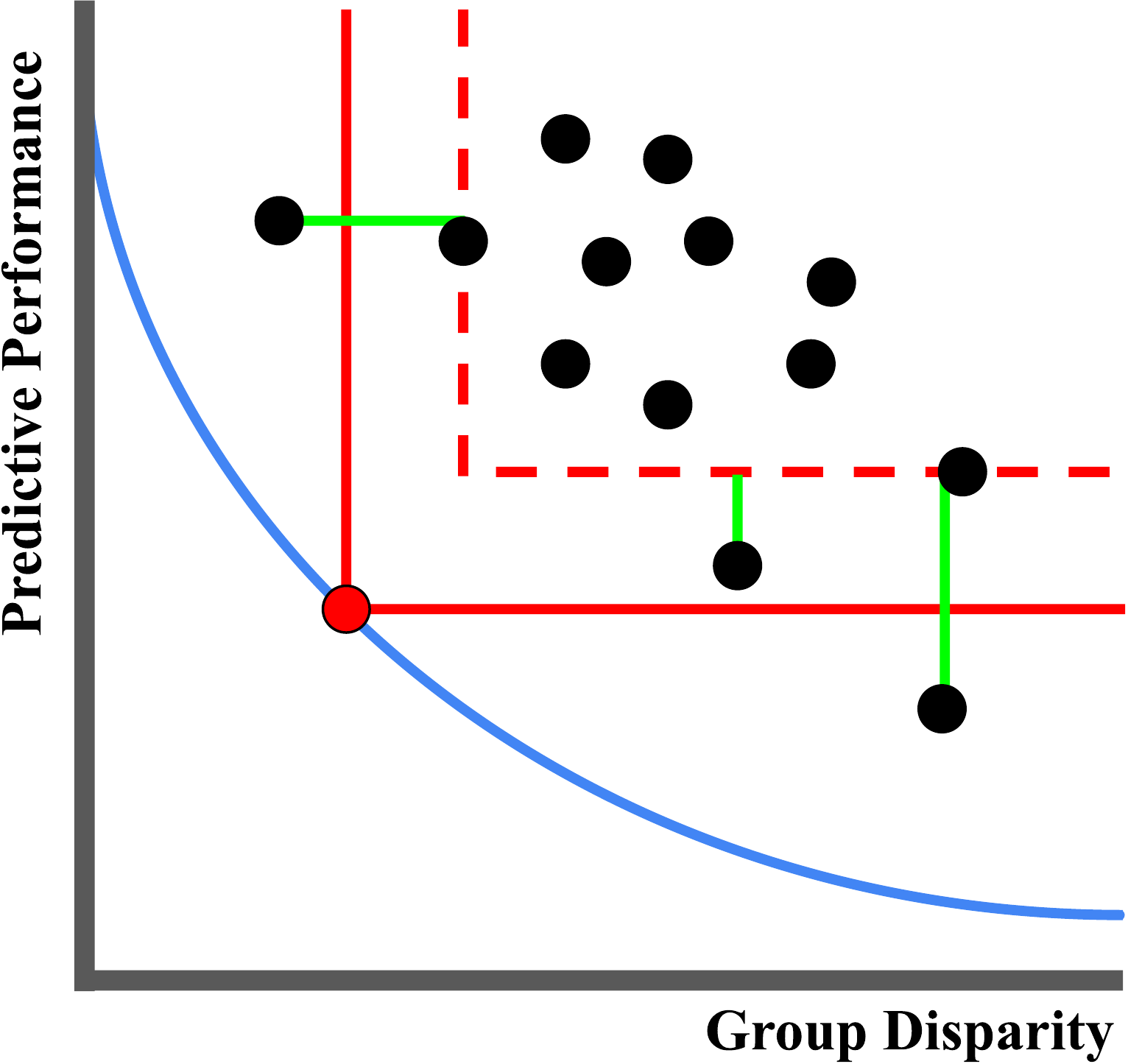} 
\end{minipage}
\begin{minipage}{0.38\columnwidth}
\vspace{-4mm}
\captionof{figure}{A Pareto frontier for possible $\hat{P}_\theta$ (blue) optimally trading off predictive performance (e.g., inaccuracy) and group unfairness.
The model-produced decision (red point) defines dominance boundaries (solid red) and margin boundaries (dashed red), which incur subdominance (green lines) on three examples.
}
\label{fig:subdom}
\end{minipage}
\vspace{-7mm}
\end{center}
\end{figure}

\begin{definition}
A fairness-aware classifier is considered $\gamma$-superhuman for a given set of predictive loss and unfairness measures 
$\{f_k\}$ if its decisions $\hat{\yvec}$ satisfy:
$$
P\left(\boldsymbol{f}\left(\hat{\bf y}, {\bf y}, {\bf a}\right) \preceq \boldsymbol{f}\left(\tilde{\bf y}, {\bf y}, {\bf a}\right)\right) \geq \gamma.
$$
\end{definition}

If strict Pareto dominance is required to be $\gamma$-superhuman, which is often effectively true for continuous domains, then by definition, at most $(1-\gamma)\%$ of human decision makers could be $\gamma$-superhuman.
However, far fewer than $(1-\gamma)$ may be $\gamma-$superhuman if pairs of human decisions do not Pareto dominate one another in either direction (i.e., neither $\boldsymbol{f}\left(\tilde{\bf y}, {\bf y}, {\bf a}\right) \preceq \boldsymbol{f}\left(\tilde{\bf y}', {\bf y}, {\bf a}\right)$ nor $\boldsymbol{f}\left(\tilde{\bf y}', {\bf y}, {\bf a}\right) \preceq \boldsymbol{f}\left(\tilde{\bf y}, {\bf y}, {\bf a}\right)$ for pairs of human decisions $\tilde{\bf y}$ and $\tilde{\bf y}'$).
From this perspective, any decisions with $\gamma-$superhuman performance more than $(1-\gamma)\%$ of the time exceed the performance limit for the distribution of human demonstrators.

Unfortunately, directly maximizing $\gamma$ is difficult in part due to the discontinuity of Pareto dominance ($\preceq$). 
The subdominance \cite{ziebart2022towards} serves as a convex upper bound for non-dominance in each metric $\{f_k\}$ and on $1-\gamma$ in aggregate:
{\small
\begin{align}
& \operatorname{subdom}^k_{\alpha_k}(\hat{\yvec}, \tilde{\yvec}, \yvec, \vec{a})
\triangleq 
\left[\alpha_k\left(f_k(\hat{\yvec}, \yvec, \vec{a}) -f_k(\tilde{\yvec}, \yvec, \vec{a})\right)+1 \right]_+. \notag\\
& \operatorname{subdom}_{\boldsymbol{\alpha}}(\hat{\yvec}, \tilde{\yvec}, \yvec, \vec{a})
\triangleq \sum_k \operatorname{subdom}^k_{\alpha_k}(\hat{\yvec}, \tilde{\yvec}, \yvec, \vec{a}).
\end{align}
}%
Given $N$ vectors of reference decisions as demonstrations, $\tilde{\Yvec}=\{\tilde{\yvec}_i\}_{i=1}^{\text{N}}$, 
the subdominance for decision vector $\hat{\bf y}$ with respect to the set of demonstrations is\footnote{For notational simplicity, we assume all demonstrated decisions $\tilde{\yvec} \in \tilde{\Yvec}$ correspond to the same $M$ items represented in  $\Xvec$. Generalization to unique $\Xvec$ for each demonstration is straightforward.}
\begin{align}
&\operatorname{subdom}_{\boldsymbol{\alpha}}(\hat{\yvec}, \tilde{\Yvec}, \yvec, \vec{a})= \frac{1}{N}\sum_{\tilde{\yvec} \in \tilde{\Yvec}}\operatorname{subdom}_{\boldsymbol{\alpha}}(\hat{\yvec}, \tilde{\yvec}, \yvec, \vec{a}),\notag
\end{align}
where $\hat{\yvec}_i$ is the predictions produced by our model for the set of items $\Xvec_i$, and $\hat{\Yvec}$ is the set of these prediction sets, $\hat{\Yvec}=\{\hat{\yvec}_i\}_{i=1}^{\text{N}}$.
The subdominance is illustrated by Figure \ref{fig:subdom}.
Following concepts from support vector machines \cite{cortes1995support}, reference decisions $\tilde{\yvec}$ that actively constrain the predictions $\hat{\yvec}$ in a particular feature dimension, k, are referred to as \emph{support vectors} and denoted as: 
\[
\tilde{\Yvec}_{\text{SV}_k}(\hat{\yvec}, \alpha_k) = \Bigl\{ \tilde{\yvec} | \alpha_k(f_k(\hat{\yvec})-f_k(\tilde{\yvec})) + 1 \geq 0 \Bigr\}.
\]

\subsection{Performance-Fairness Subdominance Minimization}

We consider probabilistic predictors $\mathbb{P}_\theta : \mathcal{X}^M \rightarrow \Delta_{\mathcal{Y}^M}$ that make structured predictions over the set of items in the most general case, but can also be simplified to make conditionally independent decisions for each item.

\begin{definition}\label{def:min-sub-opt}
The minimally subdominant fairness-aware classifier 
$\hat{\Pbb}_\vec{\theta}$ has model parameters $\vec{\theta}$ chosen by:
{\small
\begin{align}
\operatornamewithlimits{argmin}_{\vec{\theta}} \min_{\vec{\alpha} \succeq 0}  \mathbb{E}_{\hat{\yvec}|{\bf X} \sim P_\vec{\theta}} \left[ 
\operatorname{subdom}_{\vec{\alpha}, 1}\left(\hat{\yvec}, \tilde{\Yvec}, \yvec, \vec{a}\right)\right]+\lambda\|\vec{\alpha}\|_1 . \notag
\end{align}}%
\end{definition}
 Hinge loss slopes $\vec{\alpha} \triangleq \{\alpha_k\}_{k=1}^{\text{K}}$ 
 are also learned from the data during training. For subdominance of $k_{\text{th}}$ feature, $\alpha_k$ indicates the degree of sensitivity to how much the algorithm fails to sufficiently outperform demonstrations in that feature. When $\alpha_k$ value is higher, the algorithm chooses that feature to minimize subdominance. In our setting, features are loss/violation metrics defined to measure the performance/fairness of a set of reference decisions. 


We use the subgradient of subdominance with respect to $\vec{\theta}$ and $\vec{\alpha}$ to update these variables iteratively, and after convergence, the best learned weights $\vec{\theta}^*$ are used in the final model $\hat{\Pbb}_\vec{\theta^*}$. A commonly used model like logistic regression can be used for $\hat{\Pbb}_\vec{\theta}$.

\begin{theorem}\label{theorem:grad-theta}
The gradient of expected subdominance under
$\mathbb{P}_\theta$ 
with respect to the set of reference decisions $\{\tilde{\yvec}_i\}_{i=1}^{\text{N}}$ is:
\small
\begin{align}
&\nabla_\theta \mathbb{E}_{\hat{\bf y}|\vec{X} \sim \hat{P}_\theta}\left[\sum_k \overbrace{\min_{\alpha_k} \left(
\operatorname{subdom}^k_{\alpha_k}\left(\hat{\bf y}, \tilde{\Yvec}, \yvec, \vec{a}\right)+\lambda_k\alpha_k\right)}^{\Gamma_k\left(\hat{\yvec}, \tilde{\Yvec}, \yvec, \vec{a}\right)}\right] \notag\\
&=\mathbb{E}_{\hat{\bf y}|{\bf X} \sim \hat{P}_\theta}\bigg[\left(\sum_k \Gamma_k(\hat{\bf y}, \tilde{\Yvec}, \yvec, \vec{a})\right) \nabla_\theta\log \hat{\Pbb}_\theta(\hat{\bf y}|{\Xvec})\bigg], \notag
\end{align}
\normalsize
where the optimal $\alpha_k$ for each $\gamma_k$ is obtained from:
\begin{align}
\alpha_k=\underset{\alpha_k^{(m)}}{\argmin} \, m \mbox{ such that } f_k\left(\hat{\yvec}\right) + \lambda \leq \frac{1}{m} \sum_{j=1}^m f_k\left(\tilde{\yvec}^{(j)}\right),\notag
\end{align}
using $\alpha_k^{(j)}=\frac{1}{f_k\left(\hat{\yvec}^{(j)}\right)-f_k\left(\tilde{\yvec}^{(j)}\right)}$ to represent the $\alpha_k$ value that would make the demonstration with the $j_{th}$ smallest $f_k$ feature, $\tilde{\yvec}^{(j)}$, a support vector with zero subdominance.
\end{theorem}
Using gradient descent, we update the model weights $\vec{\theta}$ using an approximation of the gradient based on a set of sampled predictions $\hat{\yvec} \in \hat{\Yvec}$ from the model $\hat{\Pbb}_\vec{\theta}$:
\small
\begin{align}
    \vec{\theta} \leftarrow \vec{\theta} + \eta \left( \sum_{\hat{\yvec} \in \hat{\Yvec}} \left(\sum_k \Gamma_k(\hat{\bf y}, \tilde{\Yvec}, \yvec, \vec{a})\right) 
    \nabla_\theta\log \hat{\Pbb}_\theta(\hat{\yvec}|{\Xvec})\right), \notag
\end{align}
\normalsize

\begin{figure*}[t]
    \includegraphics[width=.34\textwidth]{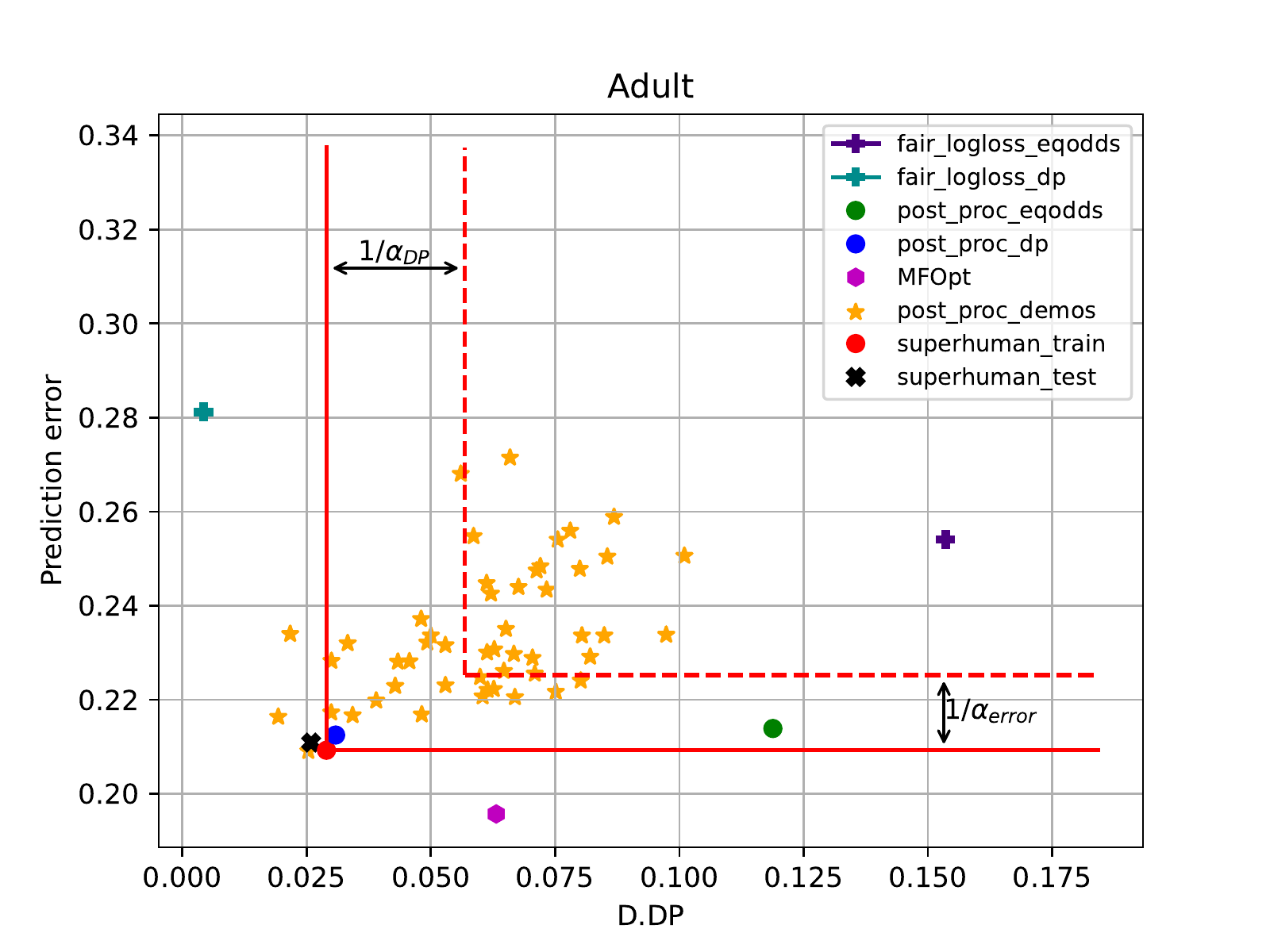} 
    \includegraphics[width=.34\textwidth]{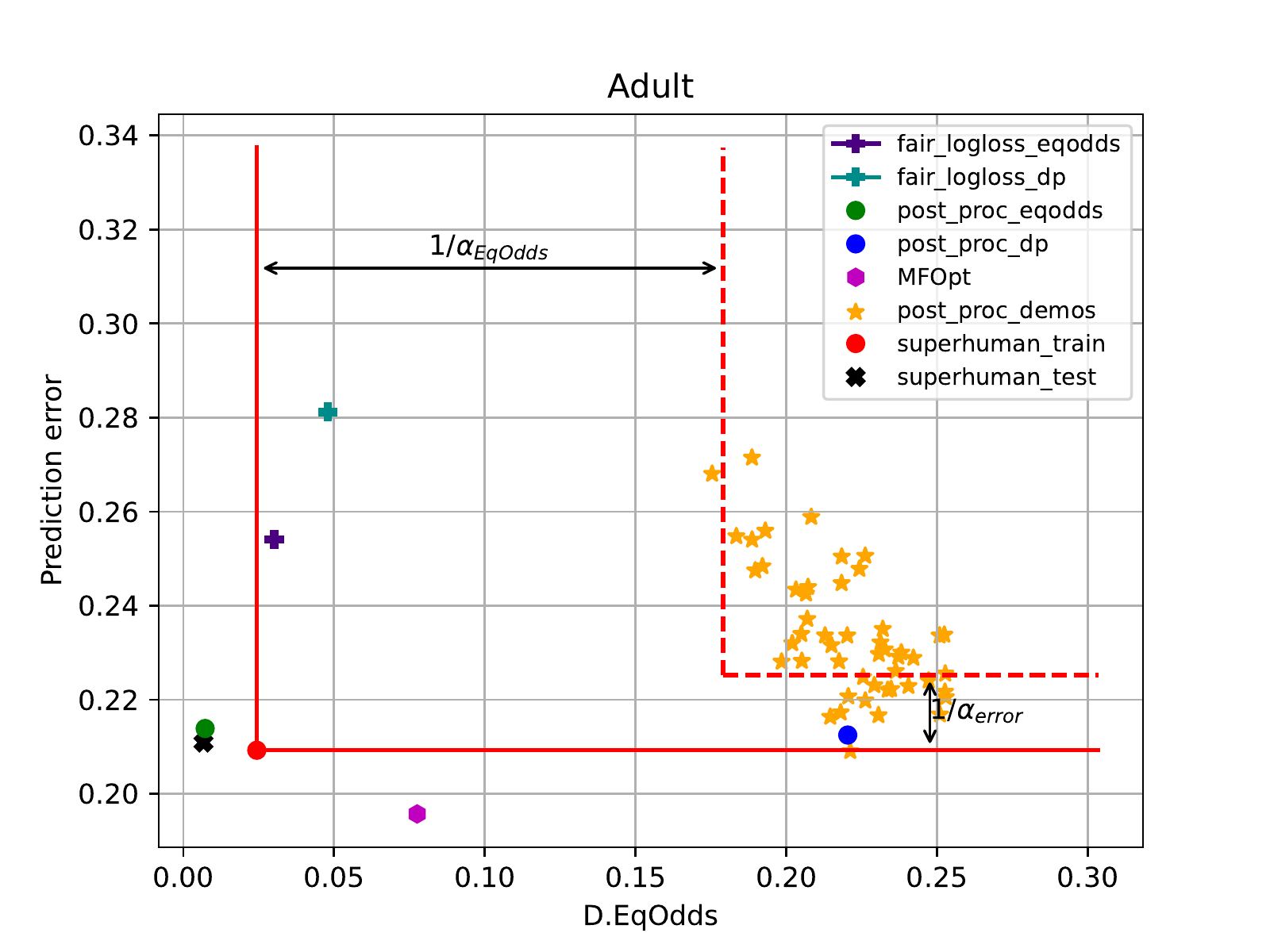}
    \includegraphics[width=.34\textwidth]{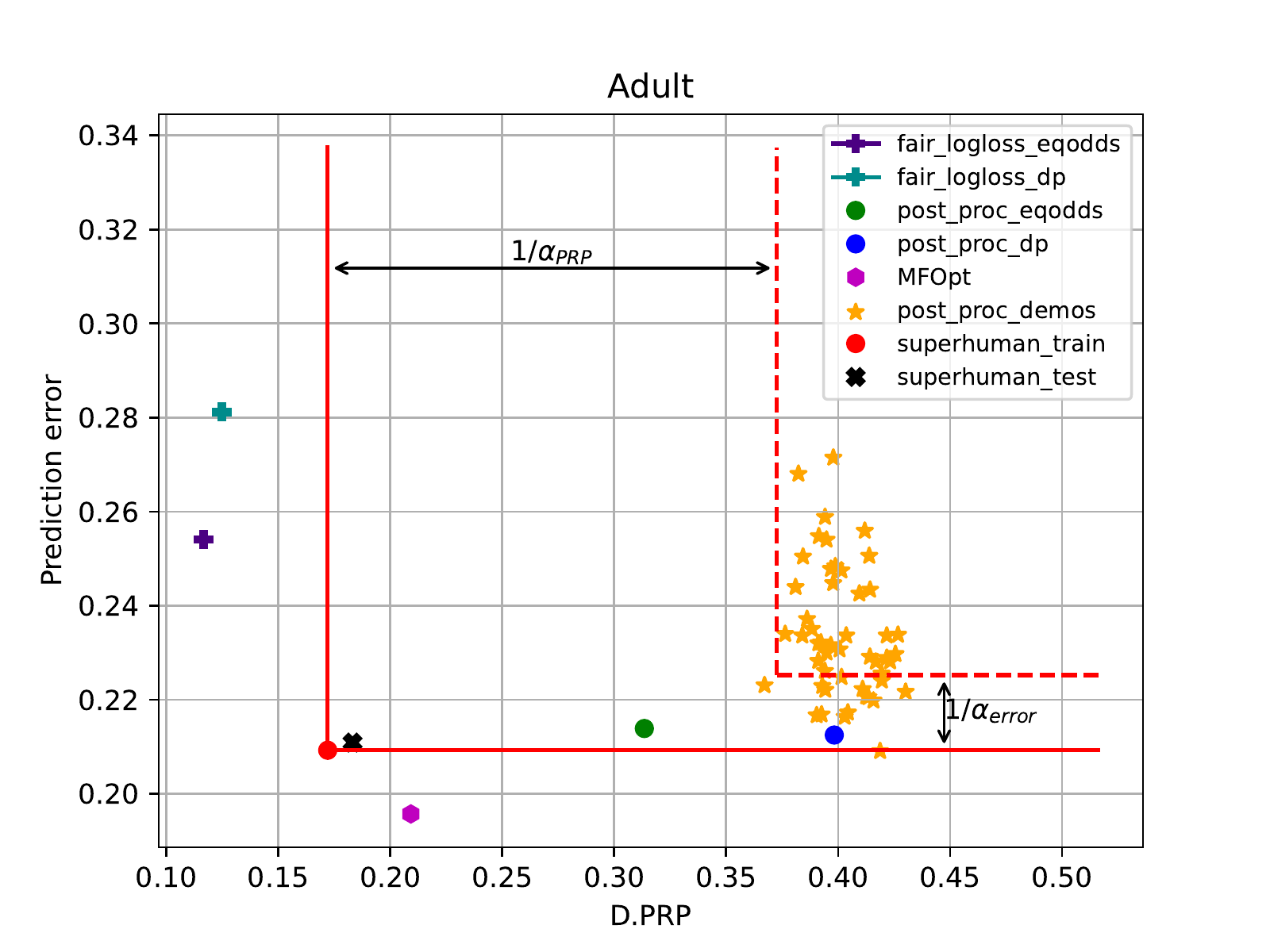}\\
    \includegraphics[width=.34\textwidth]{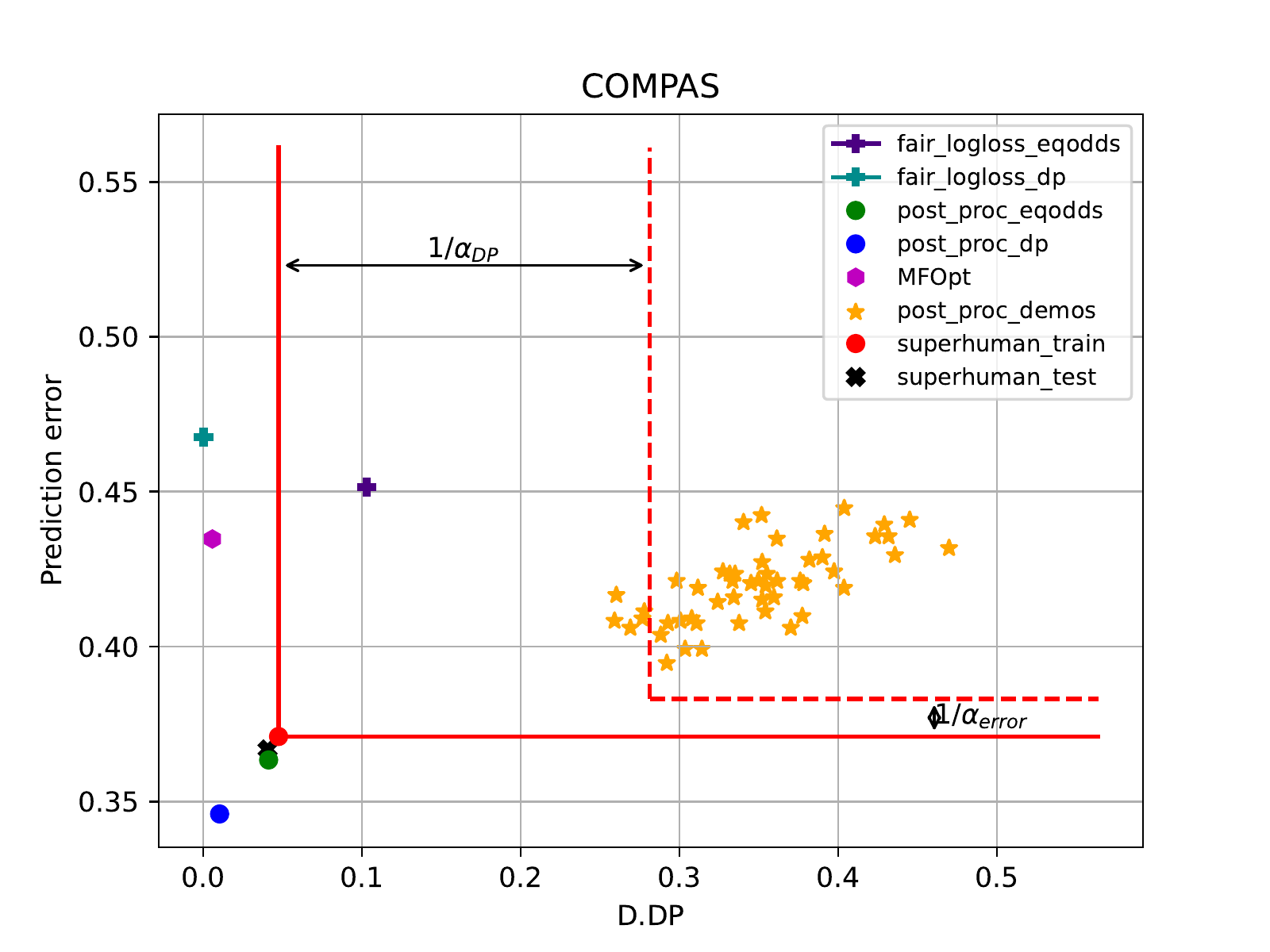} 
    \includegraphics[width=.34\textwidth]{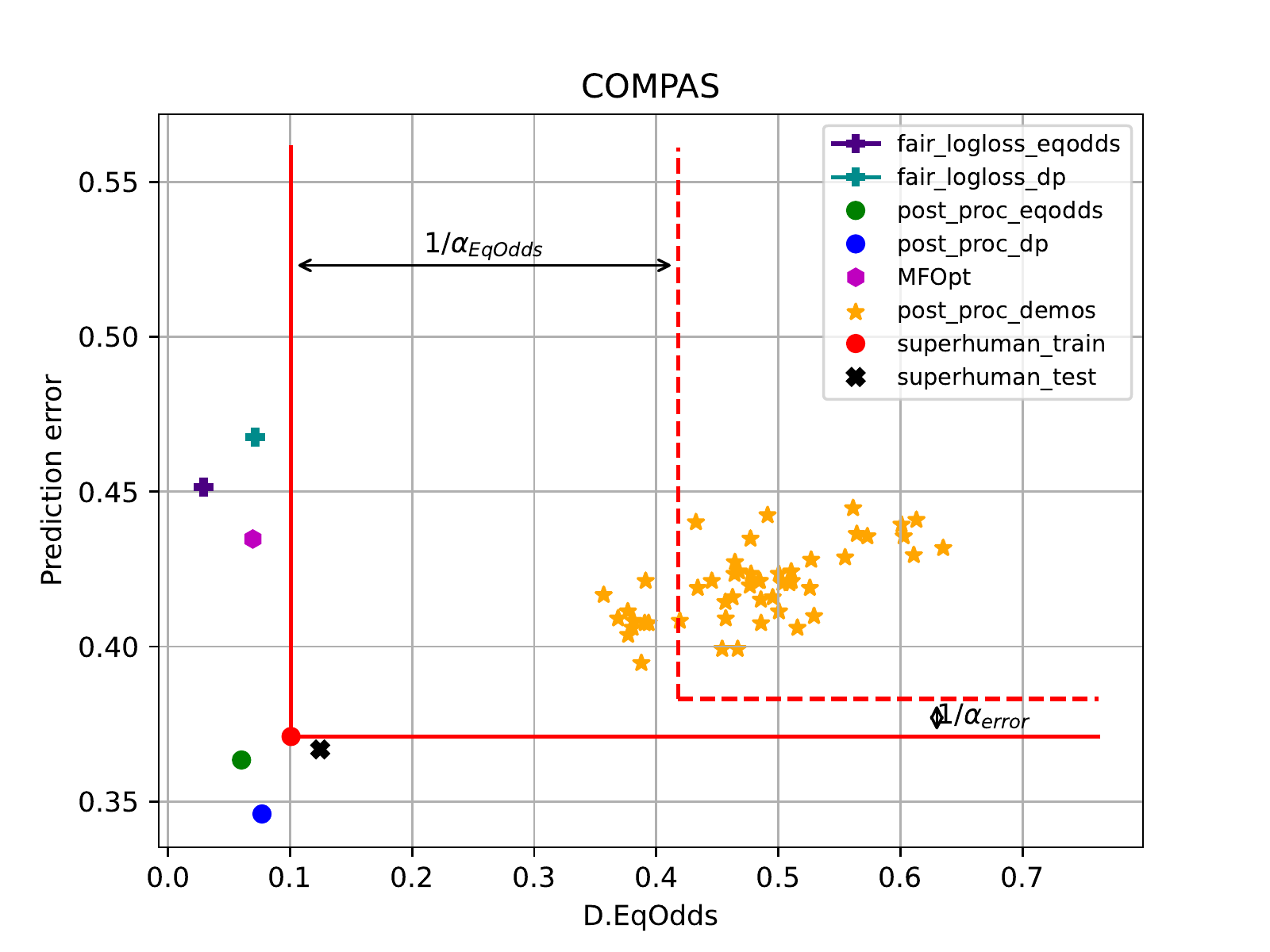}
    \includegraphics[width=.34\textwidth]{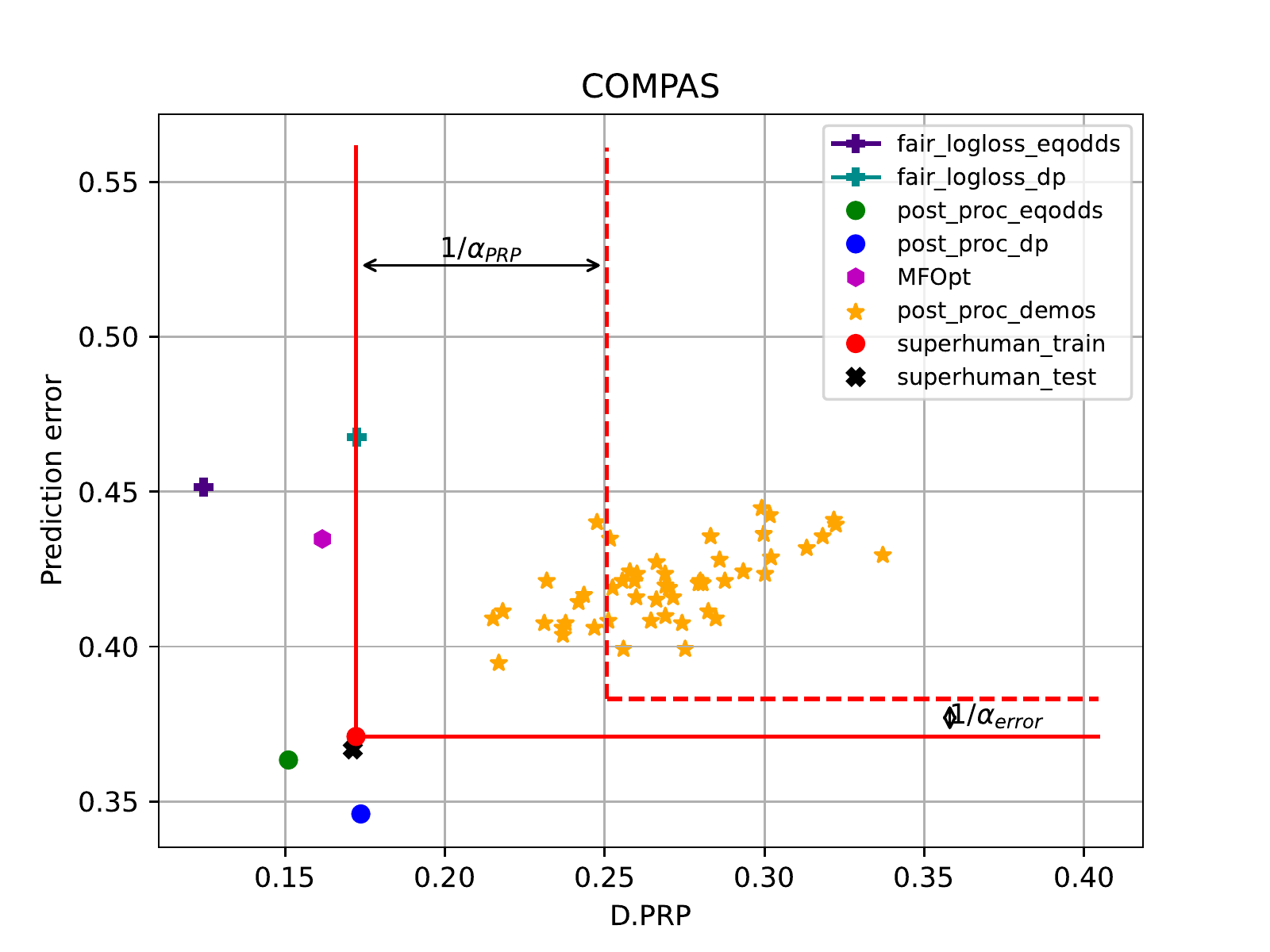}

    \caption{\small
    \emph{Prediction error} versus {\it difference of}: \emph{Demographic Parity} ({\texttt{D.DP}}), \emph{Equalized Odds} ({\texttt{D.EqOdds}}) and \emph{Predictive Rate Parity} ({\texttt{D.PR}}) on test data using noiseless training data ($\epsilon=0$) for \texttt{Adult} (top row) and \texttt{COMPAS} (bottom row) datasets.}
    \label{fig:result1}
\end{figure*}
\R{it seems that the regularization term $\lambda \Vert \alpha \Vert^2$ has been dropped from the objective in calculations of the optimal $\alpha$?}

We show the steps required for the training of our model in Algorithm \ref{alg:train}. 
\emph{Reference decisions} $\{\tilde{\yvec}_i\}_{i=1}^{\text{N}}$ from a human decision-maker or baseline method $\tilde{\Pbb}$ are provided as input to the algorithm. \, $\vec{\theta}$ is set to an initial value. 
In each iteration of the algorithm, we first sample a set of \emph{model predictions} $\{\hat{\yvec}_i\}_{i=1}^{\text{N}}$ from $\hat{\Pbb}_\vec{\theta}(.|\Xvec_i)$ for the matching items used for \emph{reference decisions} $\{\tilde{\yvec}_i\}_{i=1}^{\text{N}}$. We then find the new $\vec{\theta}$ (and $\vec{\alpha}$) based on the algorithms discussed in Theorem \ref{theorem:grad-theta}.

\vspace{-2mm}
\begin{algorithm}

Draw N set of reference decisions $\{\tilde{\yvec}_i\}_{i=1}^{\text{N}}$ from a human decision-maker or baseline method $\tilde{\Pbb}$.
Initialize: $\vec{\theta}
\leftarrow \vec{\theta}_0$\;
\While{$\vec{\theta}$ not converged}{
    Sample model predictions $\{\hat{\yvec}_i\}_{i=1}^{\text{N}}$ from $\hat{\Pbb}_\vec{\theta}(.|\Xvec_i)$ for the matching items used in reference decisions $\{\tilde{\yvec}_i\}_{i=1}^{\text{N}}$\;
  \For{$k \in \{1,...,K\}$}{
   Sort reference decisions $\{\tilde{\yvec}_i\}_{i=1}^{\text{N}}$ in ascending order based on their $k^{\text{th}}$ feature value $f_k(\tilde{\yvec}_i)$: \,  $\{\tilde{\yvec}^{(j)}\}_{j=1}^{\text{N}}$\; 
   Compute $\alpha_k^{(j)}=\frac{1}{f_k\left(\tilde{\yvec}^{(j)}\right)-f_k\left(\hat{\yvec}^{(j)}\right)}$\; 
   $\alpha_k=\underset{\alpha_k^{(m)}}{\argmin } \, m$\\ \text{ such that } $f_k\left(\hat{\yvec}^{(j)}\right) \leq \frac{1}{m} \sum_{j=1}^m f_k\left(\tilde{\yvec}^{(j)}\right)$\;
   Compute $\Gamma_k(\hat{\bf y}_i, \tilde{\Yvec}, \yvec, \vec{a})$\;
   }
   {\small
   $\vec{\theta} \leftarrow \vec{\theta} + \frac{\eta}{N}\sum_{i} \left(\sum_k \Gamma_k(\hat{\bf y}_i, \tilde{\Yvec}, \yvec, \vec{a}) \right) \nabla_\theta\log\hat{\Pbb}_\theta(\hat{\yvec}_i|{\Xvec}_i)$;
   }
 }
 \caption{Subdominance policy gradient optimization}
 \label{alg:train}
\end{algorithm}

\vspace{-3mm}
\subsection{Generalization Bounds}

With a small effort, we extend the generalization bounds based on support vectors developed for inverse optimal control subdominance minimization \cite{ziebart2022towards}.

\begin{theorem} \label{thm:generalization}
A classifier $\hat{\Pbb}_\vec{\theta}$ trained to minimize  $\operatorname{subdom}_{\boldsymbol{\alpha}}\left(\hat{\yvec}, \tilde{\yvec}_i\right)$ on a set of $N$ iid reference decisions has the support vector set $\left\{\bigcup_{ \hat{\yvec} : P_\theta(\hat{\yvec}|\Xvec) > 0} \tilde{\Yvec}_{\text{SV}_k}\left(\hat{\yvec}, \alpha_k\right)\right\}$ defined by the union of support vectors for any decision with support under $\hat{\Pbb}_\vec{\theta}$. Such a classifier is on average $\gamma$-{\bf superhuman} on the population distribution with: $\gamma=1-$ $\frac{1}{N}\left\|\bigcup_{k=1}^K \bigcup_{ \hat{\yvec} : P_\theta(\hat{\yvec}|\Xvec) > 0} \tilde{\Yvec}_{\text{S V}_k}\left(\hat{\yvec}, \alpha_k\right)\right\|$.
\end{theorem}

This generalization bound requires overfitting to the training data so that the $\hat{\Pbb}_\vec{\theta}$ has restricted support (i.e., $\hat{\Pbb}_\vec{\theta}(\hat{\yvec}|\Xvec) = 0$ for many $\hat{\yvec}$) or becomes deterministic.
\vspace{-2mm}
\section{Experiments}
\label{sec:exps}
\begin{figure*}[t]
    \includegraphics[width=.34\textwidth]{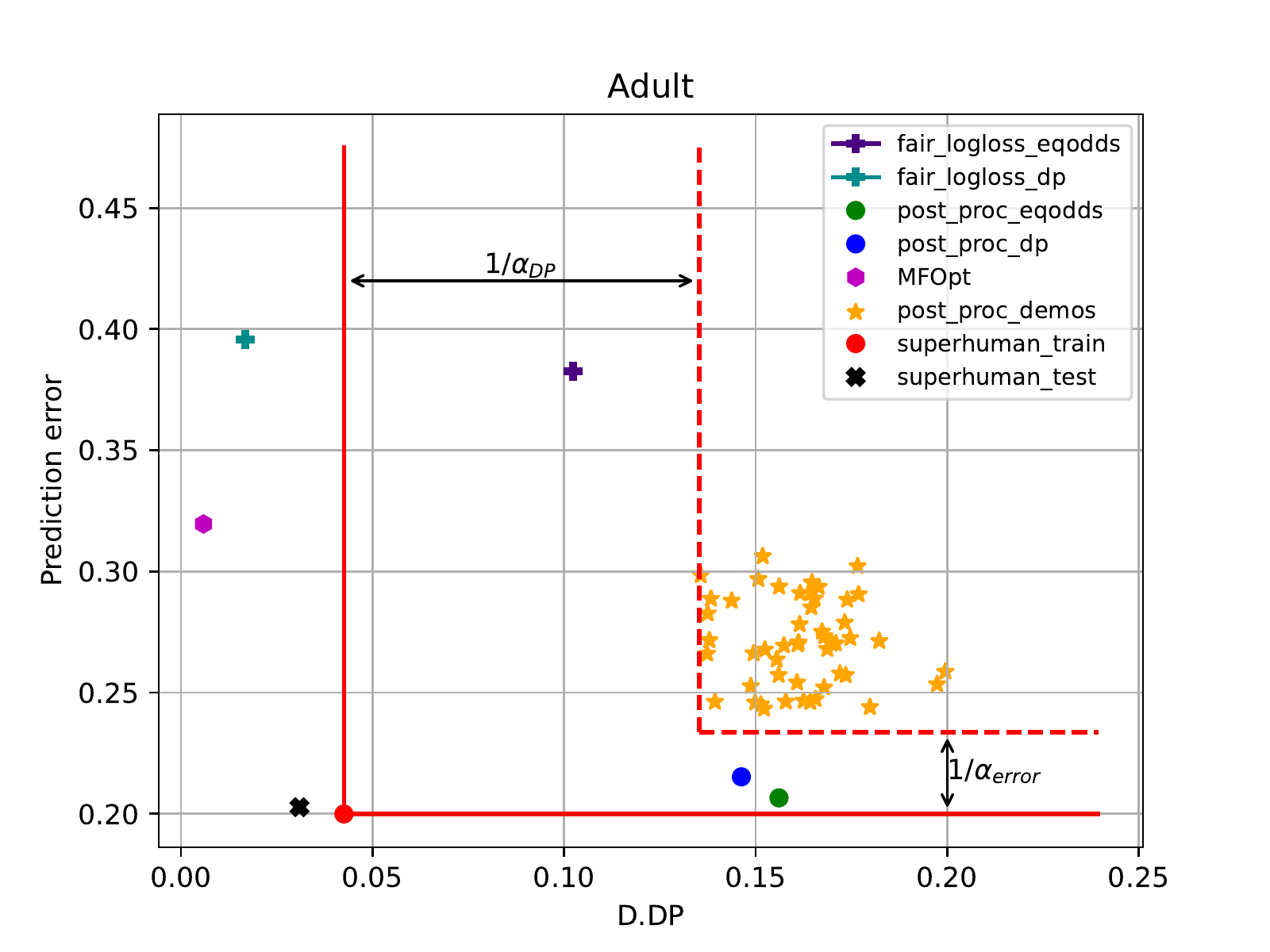} 
    \includegraphics[width=.34\textwidth]{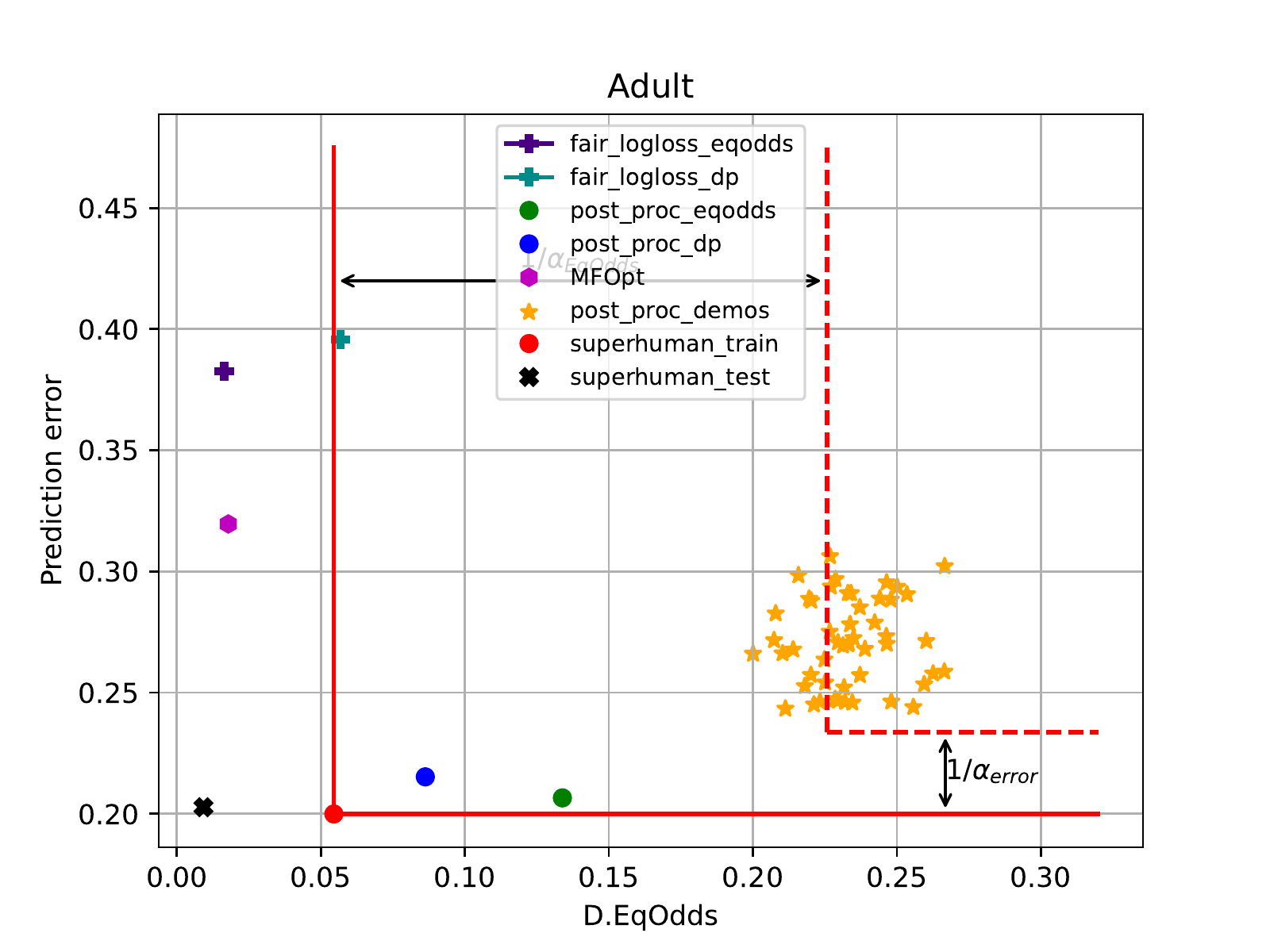}
    \includegraphics[width=.34\textwidth]{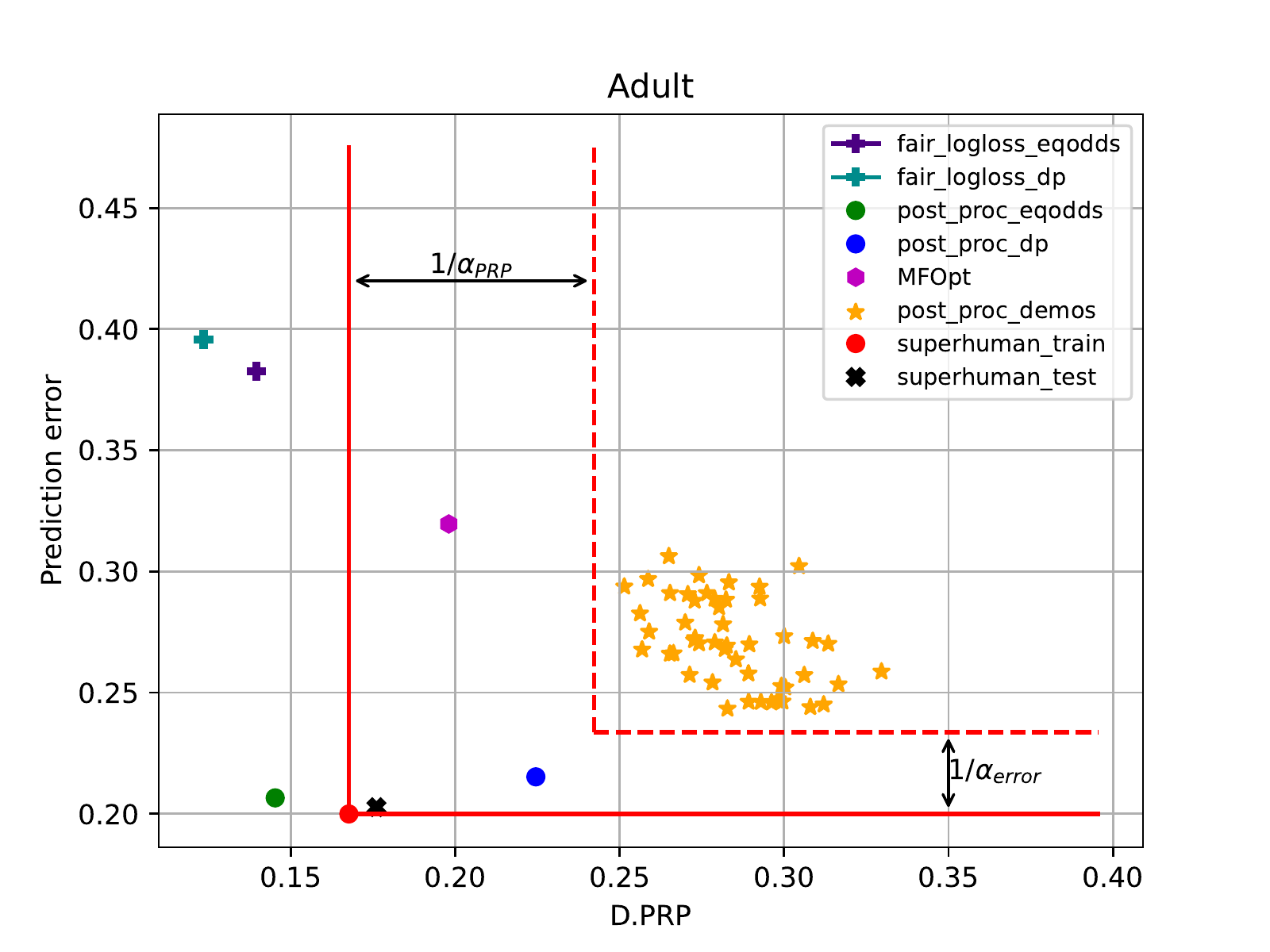}\\
    \includegraphics[width=.34\textwidth]{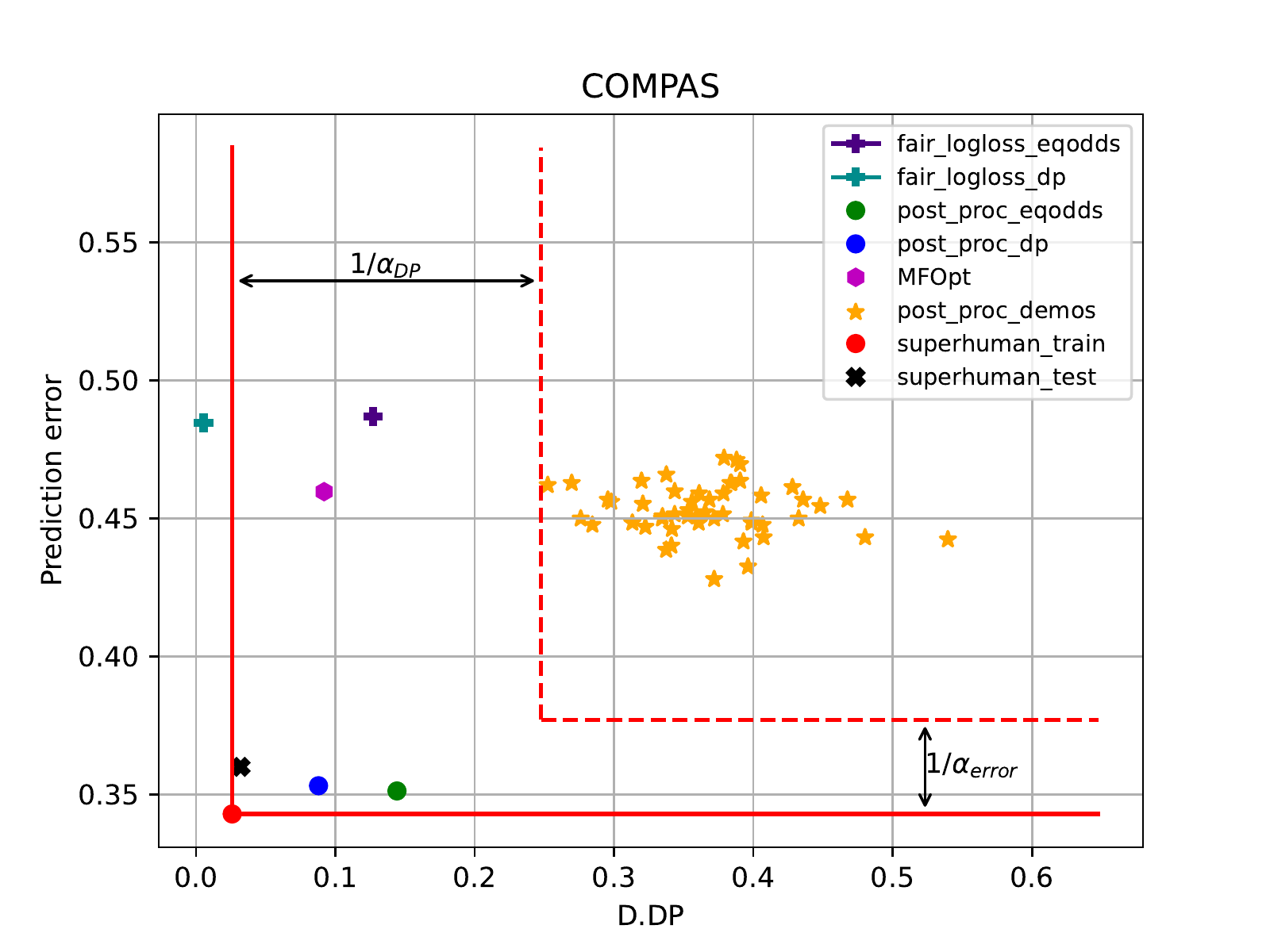} 
    \includegraphics[width=.34\textwidth]{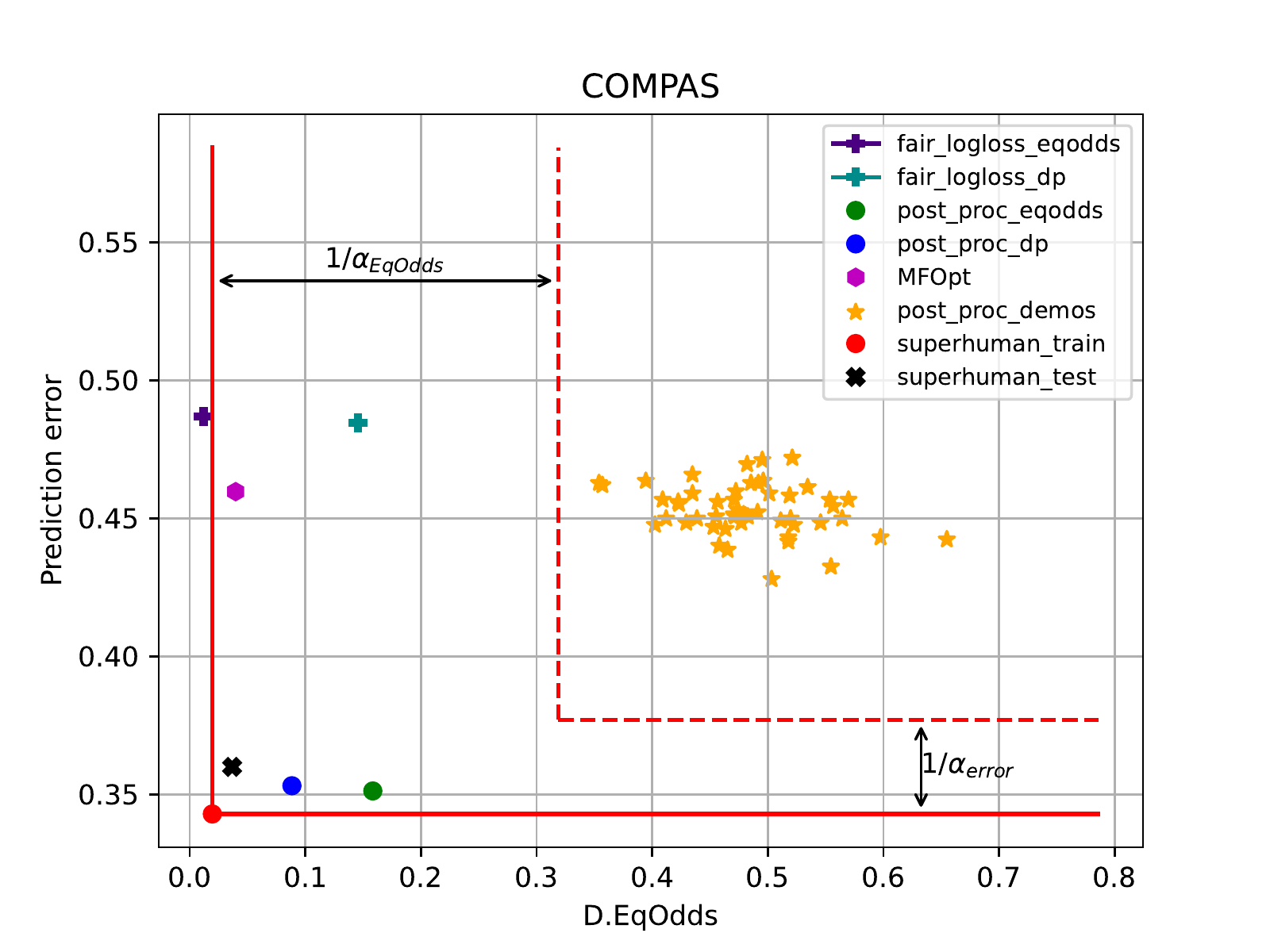}
    \includegraphics[width=.34\textwidth]{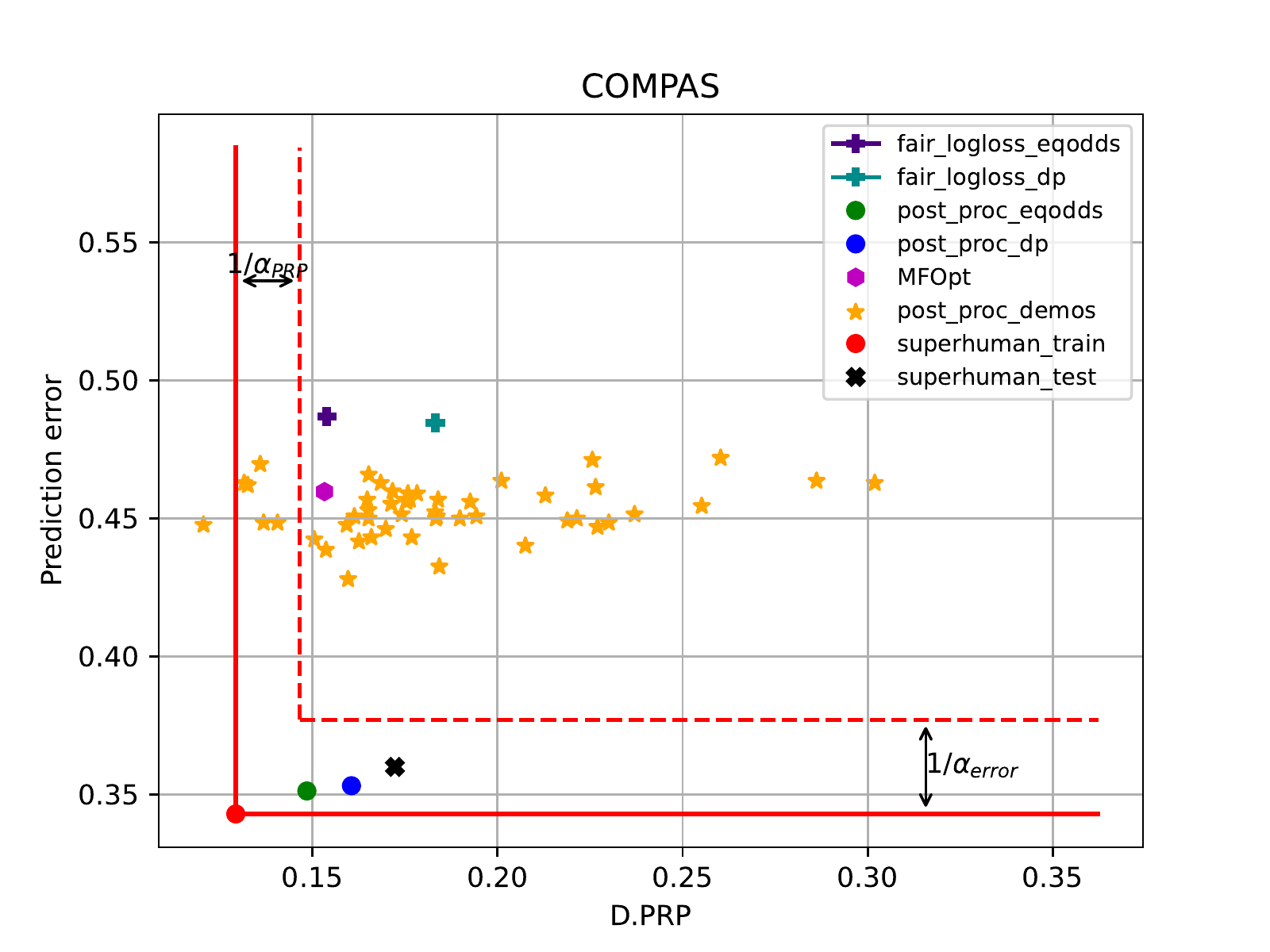}
    \caption{\small
    Experimental results on the {\tt Adult} and {\tt COMPAS} datasets with noisy demonstrations ($\epsilon=0.2$). Margin boundaries are shown with dotted red lines. Each plot shows the relationships between two features.}
    \label{fig:result2}
\end{figure*}

The goal of our approach is to produce a fairness-aware prediction method that outperforms reference (human) decisions on multiple fairness/performance measures.
In this section, we discuss our experimental design to synthesize reference decisions with varying levels of noise, evaluate our method, and provide comparison baselines.
\subsection{Training and Testing Dataset Construction}
To emulate human decision-making with various levels of noise, we add noise to the ground truth data of benchmark fairness datasets and apply fair learning methods over repeated randomized dataset splits. 
We describe this process in detail in the following section.
\vspace{-2mm}
\paragraph{Datasets} We perform experiments on 
two benchmark fairness datasets:
\begin{itemize}
\item 
UCI {\tt Adult} dataset \cite{Uci2017} considers predicting whether a household's income is higher than \$50K/yr based on census data. Group membership is based on gender. The dataset consists of 45,222 items.
\item {\tt COMPAS} dataset \cite{larson2016we} considers predicting recidivism with group membership based on race. It consists of 6,172 examples.
\end{itemize}

\vspace{-2mm}
\paragraph{Partitioning the data}
We first split entire dataset randomly into a disjoint train ({\tt train-sh}) and test ({\tt test-sh}) set of equal size. The test set ({\tt test-sh}) is entirely withheld from the training procedure and ultimately used solely for evaluation. To produce each demonstration (a vector of reference decisions), we split the ({\tt train-sh}) set, randomly into a disjoint train ({\tt train-pp}) and test ({\tt test-pp}) set of equal size.
\vspace{-2mm}
\paragraph{Noise insertion}
We randomly flip $\epsilon\%$ of the ground truth labels ${\bf y}$ and group membership attributes ${\bf a}$ to add noise to our demonstration-producing process. 
\vspace{-2mm}
\paragraph{Fair classifier $\tilde{\Pbb}$:} 
Using the noisy data, we provide existing fairness-aware methods with 
labeled {\tt train-pp} data and unlabeled {\tt test-pp}
to produce decisions on the {\tt test-pp} data as demonstrations $\tilde{\bf y}$.
Specifically, we employ:
\begin{itemize}
\item The {\bf Post-processing} method of \citet{hardt2016equality}, which aims to reduce both \emph{prediction error} and \{\emph{demographic parity} or \emph{equalized odds}\} at the same time. 
We use \emph{demographic parity} as the fairness constraint. We produce demonstrations using this method for {\tt Adult} dataset.
\item {\bf Robust fairness for logloss-based classification} \cite{rezaei2020fairness} employs distributional robustness to match target fairness constraint(s) while robustly minimizing the log loss. We use \emph{equalized odds} as the fairness constraint. We employ this method to produce demonstrations for {\tt COMPAS} dataset.


\end{itemize}
We repeat the process of partitioning {\tt train-sh} $\texttt{N}=50$ times to create randomized partitions of {\tt train-pp} and 
{\tt test-pp}
and then produce a set of demonstrations $\{\tilde{\yvec}\}_{i=1}^{50}$. 

\subsection{Evaluation Metrics and Baselines}
\paragraph{Predictive Performance and Fairness Measures}
Our focus for evaluation is on outperforming demonstrations in multiple fairness and performance measures. We use $K=4$ measures: \emph{inaccuracy} ({\tt Prediction error}), \emph{difference from demographic parity} ({\tt D.DP}), 
\emph{difference from equalized odds} ({\tt D.EqOdds}), \emph{difference from predictive rate parity} ({\tt D.PRP}). 
\vspace{-2mm}
\paragraph{Baseline methods}
As baseline comparisons, we train five different models on the entire train set ({\tt train-sh}) and then evaluate them on the withheld test data ({\tt test-sh}): 
\begin{itemize}
    \item The {\bf Post-processing} model of \cite{hardt2016equality} with \emph{demographic parity} as the fairness constraint ({\tt post\_proc\_dp}).
    \item The {\bf Post-processing} model of \cite{hardt2016equality} with \emph{equalized odds} as the fairness constraint ({\tt post\_proc\_eqodds}).
    \item The {\bf Robust Fair-logloss} model of \cite{rezaei2020fairness} with \emph{demographic parity} as the fairness constraint ({\tt fair\_logloss\_dp}).
    \item The {\bf Robust Fair-logloss} model of \cite{rezaei2020fairness} \emph{equalized odds} as the fairness constraint ({\tt fair\_logloss\_eqodds}).
    \item The {\bf Multiple Fairness Optimization} framework of \citet{hsu2022pushing} which is designed to satisfy three conflicting fairness metrics (\emph{demographic parity}, \emph{equalized odds} and \emph{predictive rate parity}) to the best extent possible ({\tt MFOpt}). 
\end{itemize}

\begin{table*}[th]
\caption{Experimental results on noise-free datasets, along with the $\alpha_k$ values learned for each feature in subdominance minimization.}
\label{tab:result1}
\centering
\resizebox{\textwidth}{!}{
\begin{tabular}{cccccccccc}
\toprule
        
        \multirow{2}{*}{\diagbox{Method}{Dataset}}
        &
        \multicolumn{4}{c}{\parbox[c]{5.4cm}{\centering \small {\tt Adult}}} 
        & 
        \multicolumn{4}{c}{\parbox[c]{5.4cm}{\centering \small {\tt COMPAS}}} 
        \\
        \cmidrule(lr){3-6} \cmidrule(lr){7-10} 
\rule{0pt}{1\normalbaselineskip}
    & & \multicolumn{1}{c}{Prediction error} & \multicolumn{1}{c}{DP diff} & \multicolumn{1}{c}{EqOdds diff} & \multicolumn{1}{c}{PRP diff} & \multicolumn{1}{c}{Prediction error} & \multicolumn{1}{c}{DP diff} & \multicolumn{1}{c}{EqOdds diff} & \multicolumn{1}{c}{PRP diff} \\ \hline \rule{0pt}{1\normalbaselineskip}
$\alpha_k$                      &  & 62.62 & 35.93 & 6.46 & 4.98 & 82.5 & 4.27 & 3.15 & 12.72\\
{\tt $\gamma$-{\bf superhuman}}          &   & 98\%      & 94\%  & 100\%  & 100\% & 100\% & 100\% & 100\% & 100\%  \\ \hline
{\tt MinSub-Fair} (ours)  &    & 0.210884  & 0.025934  & {\bf 0.006690}  & 0.183138 & 0.366806 & 0.040560 & 0.124683 & 0.171177\\
{\tt MFOpt} & &   {\bf 0.195696} & 0.063152 & 0.077549 & 0.209199 & 0.434743 & 0.005830 & 0.069519 & 0.161629 \\ 
{\tt post\_proc\_dp} & & 0.212481 & 0.030853 & 0.220357 & 0.398278 & 0.345964 & 0.010383 & 0.077020 & 0.173689\\
{\tt post\_proc\_eqodds} & & 0.213873 & 0.118802 & 0.007238 & 0.313458  & {\bf 0.363395} & 0.041243 & 0.060244 &  0.151040\\
{\tt fair\_logloss\_dp} & &   0.281194 & {\bf 0.004269} & 0.047962 & 0.124797 & 0.467610 & {\bf 0.000225} & 0.071392 &  0.172418\\ 
{\tt fair\_logloss\_eqodds} & &   0.254060 & 0.153543 & 0.030141 & {\bf 0.116579} & 0.451496 & 0.103093 & {\bf 0.029085} & {\bf 0.124447} \\

\hline
\end{tabular}
}
\end{table*}

\vspace{-2mm}
\paragraph{Hinge Loss Slopes}
As discussed previously, $\alpha_k$ value corresponds to the hinge loss slope, which defines by how far a produced decision does not sufficiently outperform the demonstrations on the $k_{\text{th}}$ feature. When the $\alpha_k$ is large, the model chooses heavily weights support vector reference decisions for that particular $k$ when minimizing subdominance. We report these values in our experiments.
\vspace{-2mm}
\subsection{Superhuman Model Specification and Updates}
We use a \emph{logistic regression} model $\Pbb_{\vec{\theta}_0}$ 
with first-order moments feature function, $\phi(y, \mathbf{x}) = [x_1y, x_2y, \dots x_my]^\top$, and weights $\vec{\theta}$ applied independently on each item as our decision model.
 During the training process, we update the model parameter $\vec{\theta}$ to reduce subdominance. 
\vspace{-2mm}
\paragraph{Sample from Model $\hat{\Pbb}_{\vec{\theta}}$}
In each iteration of the algorithm, we first sample \emph{prediction vectors} $\{\hat{\yvec}_i\}_{i=1}^{\text{N}}$ from $\hat{\Pbb}_\vec{\theta}(.|\Xvec_i)$ for the matching items used in demonstrations $\{\tilde{\yvec}_i\}_{i=1}^{\text{N}}$. In the implementation, to produce the $i_{\text{th}}$ sample, we look up the indices of the items used in $\tilde{\yvec}_i$, which constructs item set $\Xvec_i$. Now we make predictions using our model on this item set $\hat{\Pbb}_\vec{\theta}(.|\Xvec_i)$. The model produces a probability distribution for each item which can be sampled and used as a prediction $\{\hat{\yvec}_i\}_{i=1}^{\text{N}}$.

\paragraph{Update model parameters $\vec{\theta}$}
We update $\vec{\theta}$ until convergence using Algorithm \ref{alg:train}. 
For our logistic regression model, the gradient is:
\begin{align}
    \nabla_\theta \log \hat{\Pbb}_\theta(\hat{\yvec}|\Xvec) = \phi(\hat{\yvec},\Xvec) - \mathbb{E}_{\hat{\yvec}|\Xvec \sim \hat{P}_\theta}\left[\phi(\hat{\bf y},\Xvec)\right],\notag
\end{align}
where $\phi$ denotes the feature function and $\phi(\hat{\yvec},\Xvec)=\sum_{m=1}^{\text{M}} \phi(\hat{\yvec}_{m},\xvec_{m})$ is the corresponding feature function for the $i^{\text{th}}$ set of reference decisions.

\OM{Mention that we choose learning rate as $\eta = 0.01$}

\begin{table*}[th]
\caption{Experimental results on datasets with noisy demonstrations, along with the $\alpha_k$ values learned for each feature.}
\label{tab:result2}
\centering
\resizebox{\textwidth}{!}{
\begin{tabular}{cccccccccc}
\toprule
        
        \multirow{2}{*}{\diagbox{Method}{Dataset}}
        &
        \multicolumn{4}{c}{\parbox[c]{5.4cm}{\centering \small {\tt Adult}}} 
        & 
        \multicolumn{4}{c}{\parbox[c]{5.4cm}{\centering \small {\tt COMPAS}}} 
        \\
        \cmidrule(lr){3-6} \cmidrule(lr){7-10} 
\rule{0pt}{1\normalbaselineskip}
    & & \multicolumn{1}{c}{Prediction error} & \multicolumn{1}{c}{DP diff} & \multicolumn{1}{c}{EqOdds diff} & \multicolumn{1}{c}{PRP diff} & \multicolumn{1}{c}{Prediction error} & \multicolumn{1}{c}{DP diff} & \multicolumn{1}{c}{EqOdds diff} & \multicolumn{1}{c}{PRP diff} \\ \hline \rule{0pt}{1\normalbaselineskip}
$\alpha_k$                      &  & 29.63 & 10.77 & 5.83 & 13.42 & 29.33 & 4.51 & 3.34 & 57.74\\
{\tt $\gamma$-{\bf superhuman}}          &   & 100\%      & 100\%  & 100\%  & 100\%  & 100\% & 100\% & 100\% & 98\% \\ \hline
{\tt MinSub-Fair} (ours)  &    & 0.202735      & 0.030961    & {\bf 0.009263}  & 0.176004 & 0.359985 & 0.031962 & 0.036680 & 0.172286\\ 
{\tt MFOpt} & &   {\bf 0.195696} & 0.063152 & 0.077549 & 0.209199 & 0.459731 & 0.091892 & 0.039745 & 0.153257 \\ 
{\tt post\_proc\_dp} & & 0.225462 & 0.064232 & 0.237852 & 0.400427 & 0.353164 & 0.087889 & 0.088414 & 0.160538 \\
{\tt post\_proc\_eqodds} & & 0.224561 & 0.103158 & 0.010552 & 0.310070 & {\bf 0.351269} & 0.144190 & 0.158372 & {\bf 0.148493}\\
{\tt fair\_logloss\_dp} & &   0.285549 & {\bf 0.007576} & 0.057659 & {\bf 0.115751} &  0.484620 & {\bf 0.005309} & 0.145502 & 0.183193 \\ 
{\tt fair\_logloss\_eqodds} & & 0.254577 & 0.147932 & 0.012778 & 0.118041 & 0.487025 & 0.127163 & {\bf 0.011918} & 0.153869 \\

\hline
\end{tabular}
}
\end{table*}

\subsection{Experimental Results}

After training each model, e.g., obtaining the best model weight $\vec{\theta}^*$ from the training data ({\tt train-sh}) for {\tt superhuman}, we evaluate each 
on unseen test data ({\tt test-sh}). 
We employ hard predictions (i.e., the most probable label) using our approach at time time rather than randomly sampling.
\begin{figure}[t]
\begin{center}
\begin{minipage}{0.9\columnwidth}
\includegraphics[width=0.9\columnwidth]{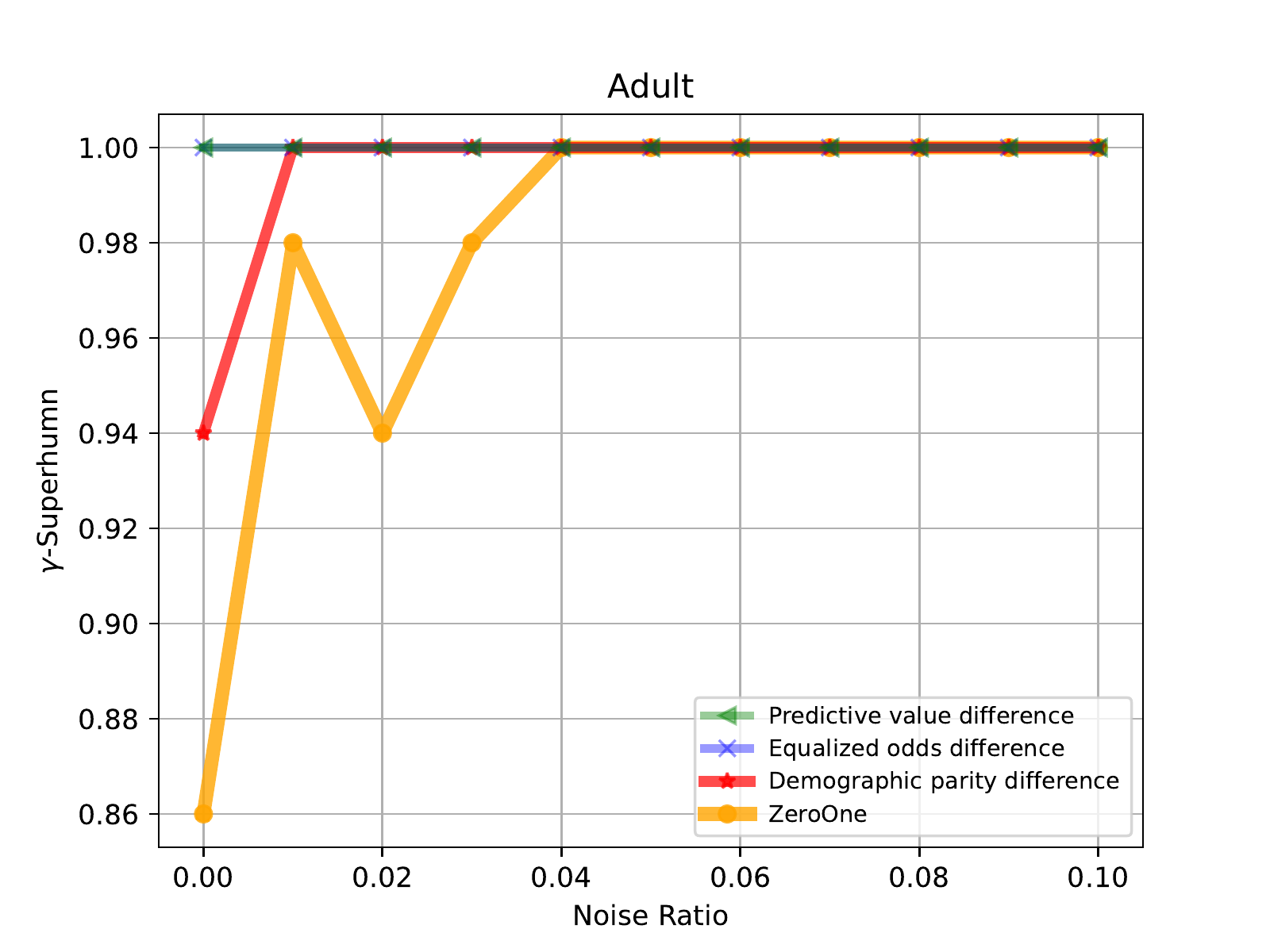} 
\end{minipage}
\begin{minipage}{0.9\columnwidth}
\includegraphics[width=0.9\columnwidth]{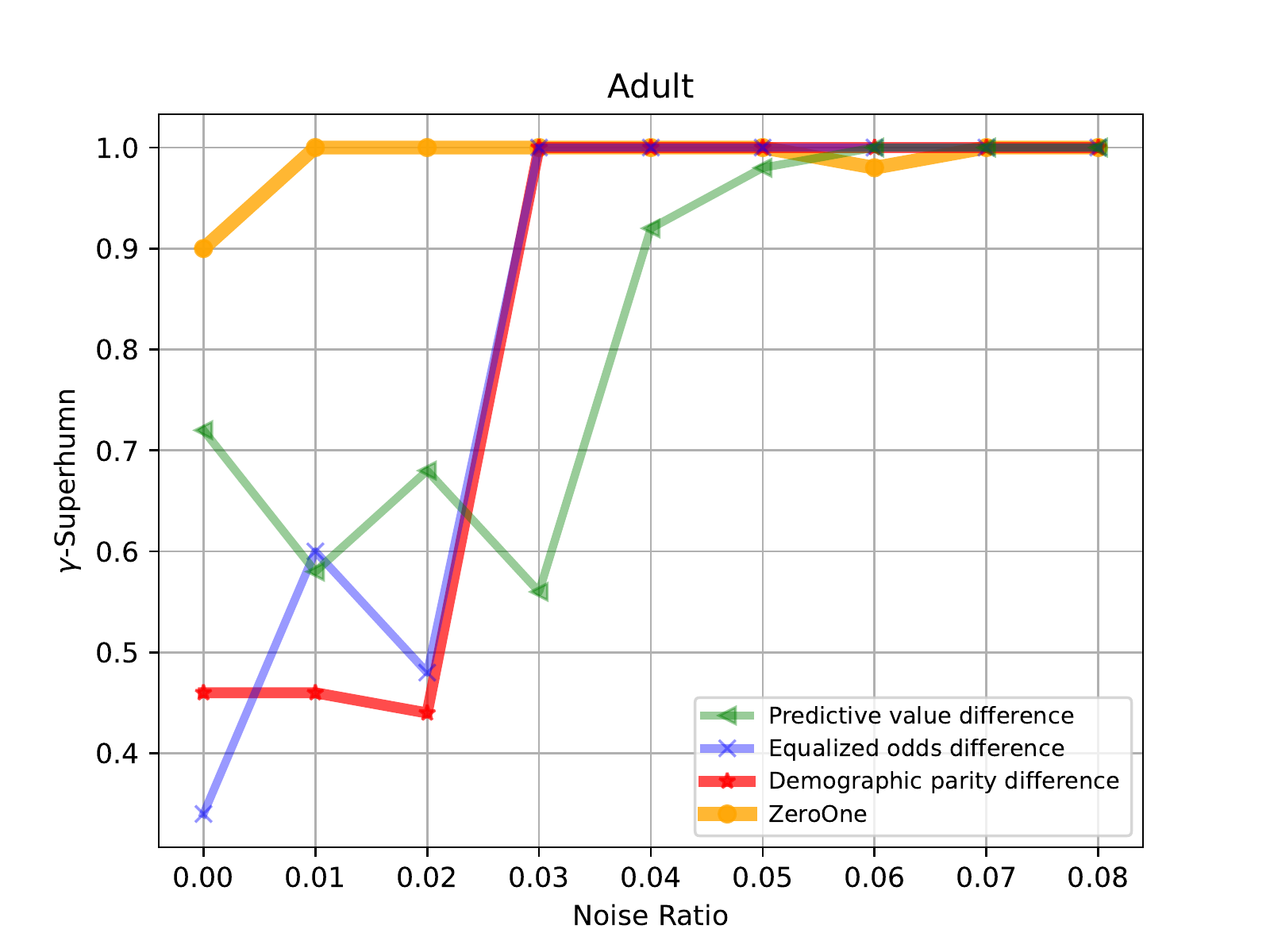}
\end{minipage}
\begin{minipage}{0.95\columnwidth}
\vspace{-2mm}
\captionof{figure}{The relationship between the ratio of augmented noise in the label and the protected attribute of reference decisions produced by post-processing (upper) and fair-logloss (lower) and achieving $\gamma$-superhuman performance in our approach.}
\label{fig:noise-gamma}
\end{minipage}
\vspace{-7mm}
\end{center}
\end{figure}

\vspace{-1mm}
\begin{table*}[th]
\vspace{-1mm}
\caption{Percentage of reference demonstrations that each method outperforms in all prediction/fairness measures.}
\centering
\begin{tabular}{ccccc}
\hline
Method    & \multicolumn{1}{c}{{\tt Adult}($\epsilon=0.0$)} & \multicolumn{1}{c}{{\tt Adult}($\epsilon=0.2$)} & \multicolumn{1}{c}{{\tt COMPAS}($\epsilon=0.0$)} & \multicolumn{1}{c}{{\tt COMPAS}($\epsilon=0.2$)} \\ \hline
{\tt MinSub-Fair} (ours)     & {\bf 96\%}   & {\bf 100\%}   & {\bf 100\%} &  {\bf 98\%} \\
{\tt MFOpt}     & 42\%  &  0\% & 18\% &  18\%  \\
{\tt post\_proc\_dp}       & 16\%  & 86\%  & {\bf 100\%} & 80\%   \\
{\tt post\_proc\_eqodds} &  0\% & 66\% & {\bf 100\%} & 88\% \\ 
{\tt fair\_logloss\_dp}  & 0\% & 0\% & 0\% & 0\% \\
{\tt fair\_logloss\_eqodds} & 0\% & 0\% & 0\% & 0\% \\
\hline
\end{tabular}
\label{tab:all_gamma}
\end{table*}

\paragraph{Noise-free reference decisions}
Our first set of experiments considers learning from reference decisions with no added noise.
The results are shown in Figure \ref{fig:result1}. We observe that our approach outperforms demonstrations in all fairness metrics and shows comparable performance in \emph{accuracy}. The ({\tt post\_proc\_dp}) performs almost as an average of demonstrations in all dimensions, hence our approach can outperform it in all fairness metrics. In comparison to ({\tt post\_proc\_dp}), our approach can outperform in all fairness metrics but is slightly worse in \emph{prediction error}. 

We show the experiment results along with $\alpha_k$ values in Table \ref{tab:result1}. 
Note that the margin boundaries (dotted red lines) in Figure \ref{fig:result1} are equal to $\frac{1}{\alpha_k}$ for feature $k$, hence there is reverse relation between $\alpha_k$ and margin boundary for feature $k$. 
We observe larger values of $\alpha_k$ for \emph{prediction error} and \emph{demographic parity difference}. The reason is that these features are already optimized in demonstrations and our model has to increase $\alpha_k$ values for those features to sufficiently outperform them.

\paragraph{Noisy reference decisions}
In our second set of experiments, we introduce significant amounts of noise ($\epsilon=0.2$) into our reference decisions.
The results for these experiments are shown in Figure \ref{fig:result2}. We observe that in the case of learning from noisy demonstrations, our approach still outperforms the reference decisions. The main difference here is that due to the noisy setting, demonstrations have worse \emph{prediction error} but regardless of this issue, our approach still can achieve a competitive \emph{prediction error}. We show the experimental results along with $\alpha_k$ values in Table \ref{tab:result2}.

\paragraph{Relationship of noise to superhuman performance} We also evaluate the relationship between the amount of augmented noise in the label and protected attribute of demonstrations, with achieving $\gamma$-superhuman performance in our approach. As shown in Figure \ref{fig:noise-gamma}, with slightly increasing the amount of noise in demonstrations, our approach can outperform $100\%$ of demonstrations and reach to $1$-superhuman performance. In Table \ref{tab:all_gamma} we show the percentage of demonstrations that each method can outperform across all prediction/fairness measures (i.e., the $\gamma-$superhuman value).

\section{Conclusions}
\label{conclusion}
In this paper, we introduce superhuman fairness,
an approach to fairness-aware classifier construction based on imitation learning. 
Our approach avoids explicit performance-fairness trade-off specification or elicitation.
Instead, it seeks to unambiguously outperform human decisions across multiple performance and fairness measures with maximal frequency.
We develop a general framework for pursuing this based on subdominance minimization \cite{ziebart2022towards} and policy gradient optimization methods \cite{sutton2018reinforcement} that enable a broad class of probabilistic fairness-aware classifiers to be learned.
Our experimental results show the effectiveness of our approach in outperforming synthetic decisions corrupted by small amounts of label and group-membership noise when evaluated using multiple fairness criteria combined with predictive accuracy. 

\paragraph{Societal impacts} By design, our approach has the potential to identify fairness-aware decision-making tasks in which human decisions can frequently be outperformed by a learned classifier on a set of provided performance and fairness measures.
This has the potential to facilitate a transition from manual to automated decisions that are preferred by all interested stakeholders, so long as their interests are reflected in some of those measures.
However, our approach has limitations.
First, when performance-fairness tradeoffs can either be fully specified (e.g., based on first principles) or effectively elicited, fairness-aware classifiers optimized for those trade-offs should produce better results than our approach, which operates under greater uncertainty cast by the noisiness of human decisions.
Second, if target fairness concepts lie outside the set of metrics we consider, our resulting fairness-aware classifier will be oblivious to them.
Third, our approach assumes human-demonstrated decision are well-intentioned, noisy reflections of desired performance-fairness trade-offs. If this is not the case, then our methods could succeed in outperforming them across all fairness measures, but still not provide an adequate degree of fairness.

\paragraph{Future directions}
We have conducted experiments with a relatively small number of performance/fairness measures using a simplistic logistic regression model.
Scaling our approach to much larger numbers of measures and classifiers with more expressive representations are both of great interest.
Additionally, we plan to pursue experimental validation using human-provided fairness-aware decisions in addition to the synthetically-produced decisions we consider in this paper.

\newpage
\bibliography{references}

\begin{thebibliography}{25}
\providecommand{\natexlab}[1]{#1}
\providecommand{\url}[1]{\texttt{#1}}
\expandafter\ifx\csname urlstyle\endcsname\relax
  \providecommand{\doi}[1]{doi: #1}\else
  \providecommand{\doi}{doi: \begingroup \urlstyle{rm}\Url}\fi

\bibitem[Abbeel \& Ng(2004)Abbeel and Ng]{abbeel2004apprenticeship}
Abbeel, P. and Ng, A.~Y.
\newblock Apprenticeship learning via inverse reinforcement learning.
\newblock In \emph{Proceedings of the International Conference on Machine
  Learning}, pp.\  1--8, 2004.

\bibitem[Blum \& Stangl(2019)Blum and Stangl]{blum2019recovering}
Blum, A. and Stangl, K.
\newblock Recovering from biased data: Can fairness constraints improve
  accuracy?
\newblock \emph{arXiv preprint arXiv:1912.01094}, 2019.

\bibitem[Boyd \& Vandenberghe(2004)Boyd and Vandenberghe]{boyd2004convex}
Boyd, S. and Vandenberghe, L.
\newblock \emph{Convex optimization}.
\newblock Cambridge University Press, 2004.

\bibitem[Calders et~al.(2009)Calders, Kamiran, and
  Pechenizkiy]{calders2009building}
Calders, T., Kamiran, F., and Pechenizkiy, M.
\newblock Building classifiers with independency constraints.
\newblock In \emph{2009 IEEE International Conference on Data Mining
  Workshops}, pp.\  13--18. IEEE, 2009.

\bibitem[Celis et~al.(2019)Celis, Huang, Keswani, and
  Vishnoi]{celis2019classification}
Celis, L.~E., Huang, L., Keswani, V., and Vishnoi, N.~K.
\newblock Classification with fairness constraints: A meta-algorithm with
  provable guarantees.
\newblock In \emph{ACM FAT*}, 2019.

\bibitem[Chen \& Pu(2004)Chen and Pu]{chen2004survey}
Chen, L. and Pu, P.
\newblock Survey of preference elicitation methods.
\newblock Technical report, EPFL, 2004.

\bibitem[Chouldechova(2017)]{chouldechova2017fair}
Chouldechova, A.
\newblock Fair prediction with disparate impact: A study of bias in recidivism
  prediction instruments.
\newblock \emph{Big data}, 5\penalty0 (2):\penalty0 153--163, 2017.

\bibitem[Cortes \& Vapnik(1995)Cortes and Vapnik]{cortes1995support}
Cortes, C. and Vapnik, V.
\newblock Support-vector networks.
\newblock \emph{Machine learning}, 20:\penalty0 273--297, 1995.

\bibitem[Dheeru \& Karra~Taniskidou(2017)Dheeru and Karra~Taniskidou]{Uci2017}
Dheeru, D. and Karra~Taniskidou, E.
\newblock {UCI} machine learning repository, 2017.
\newblock URL \url{http://archive.ics.uci.edu/ml}.

\bibitem[Hardt et~al.(2016)Hardt, Price, and Srebro]{hardt2016equality}
Hardt, M., Price, E., and Srebro, N.
\newblock Equality of opportunity in supervised learning.
\newblock \emph{Advances in Neural Information Processing Systems},
  29:\penalty0 3315--3323, 2016.

\bibitem[Hiranandani et~al.(2020)Hiranandani, Narasimhan, and
  Koyejo]{hiranandani2020fair}
Hiranandani, G., Narasimhan, H., and Koyejo, S.
\newblock Fair performance metric elicitation.
\newblock \emph{Advances in Neural Information Processing Systems},
  33:\penalty0 11083--11095, 2020.

\bibitem[Hsu et~al.(2022)Hsu, Mazumder, Nandy, and Basu]{hsu2022pushing}
Hsu, B., Mazumder, R., Nandy, P., and Basu, K.
\newblock Pushing the limits of fairness impossibility: Who's the fairest of
  them all?
\newblock In \emph{Advances in Neural Information Processing Systems}, 2022.

\bibitem[Kamishima et~al.(2012)Kamishima, Akaho, Asoh, and
  Sakuma]{kamishima2012fairness}
Kamishima, T., Akaho, S., Asoh, H., and Sakuma, J.
\newblock Fairness-aware classifier with prejudice remover regularizer.
\newblock In \emph{Joint European Conference on Machine Learning and Knowledge
  Discovery in Databases}, pp.\  35--50. Springer, 2012.

\bibitem[Kleinberg et~al.(2016)Kleinberg, Mullainathan, and
  Raghavan]{kleinberg2016inherent}
Kleinberg, J., Mullainathan, S., and Raghavan, M.
\newblock Inherent trade-offs in the fair determination of risk scores.
\newblock \emph{arXiv preprint arXiv:1609.05807}, 2016.

\bibitem[Larson et~al.(2016)Larson, Mattu, Kirchner, and Angwin]{larson2016we}
Larson, J., Mattu, S., Kirchner, L., and Angwin, J.
\newblock How we analyzed the compas recidivism algorithm.
\newblock \emph{ProPublica}, 9, 2016.

\bibitem[Liu \& Vicente(2022)Liu and Vicente]{liu2022accuracy}
Liu, S. and Vicente, L.~N.
\newblock Accuracy and fairness trade-offs in machine learning: A stochastic
  multi-objective approach.
\newblock \emph{Computational Management Science}, pp.\  1--25, 2022.

\bibitem[Martinez et~al.(2020)Martinez, Bertran, and
  Sapiro]{martinez2020minimax}
Martinez, N., Bertran, M., and Sapiro, G.
\newblock Minimax {P}areto fairness: A multi objective perspective.
\newblock In \emph{Proceedings of the International Conference on Machine
  Learning}, pp.\  6755--6764. PMLR, 13--18 Jul 2020.

\bibitem[Menon \& Williamson(2018)Menon and Williamson]{menon2018cost}
Menon, A.~K. and Williamson, R.~C.
\newblock The cost of fairness in binary classification.
\newblock In \emph{ACM FAT*}, 2018.

\bibitem[Osa et~al.(2018)Osa, Pajarinen, Neumann, Bagnell, Abbeel, Peters,
  et~al.]{osa2018algorithmic}
Osa, T., Pajarinen, J., Neumann, G., Bagnell, J.~A., Abbeel, P., Peters, J.,
  et~al.
\newblock An algorithmic perspective on imitation learning.
\newblock \emph{Foundations and Trends{\textregistered} in Robotics},
  7\penalty0 (1-2):\penalty0 1--179, 2018.

\bibitem[Rezaei et~al.(2020)Rezaei, Fathony, Memarrast, and
  Ziebart]{rezaei2020fairness}
Rezaei, A., Fathony, R., Memarrast, O., and Ziebart, B.
\newblock Fairness for robust log loss classification.
\newblock In \emph{Proceedings of the AAAI Conference on Artificial
  Intelligence}, volume~34, pp.\  5511--5518, 2020.

\bibitem[Sutton \& Barto(2018)Sutton and Barto]{sutton2018reinforcement}
Sutton, R.~S. and Barto, A.~G.
\newblock \emph{Reinforcement learning: An introduction}.
\newblock MIT press, 2018.

\bibitem[Syed \& Schapire(2007)Syed and Schapire]{syed2007game}
Syed, U. and Schapire, R.~E.
\newblock A game-theoretic approach to apprenticeship learning.
\newblock \emph{Advances in neural information processing systems}, 20, 2007.

\bibitem[Vapnik \& Chapelle(2000)Vapnik and Chapelle]{vapnik2000bounds}
Vapnik, V. and Chapelle, O.
\newblock Bounds on error expectation for support vector machines.
\newblock \emph{Neural computation}, 12\penalty0 (9):\penalty0 2013--2036,
  2000.

\bibitem[Ziebart et~al.(2022)Ziebart, Choudhury, Yan, and
  Vernaza]{ziebart2022towards}
Ziebart, B., Choudhury, S., Yan, X., and Vernaza, P.
\newblock Towards uniformly superhuman autonomy via subdominance minimization.
\newblock In \emph{International Conference on Machine Learning}, pp.\
  27654--27670. PMLR, 2022.

\bibitem[Ziebart et~al.(2008)Ziebart, Maas, Bagnell, Dey,
  et~al.]{ziebart2008maximum}
Ziebart, B.~D., Maas, A.~L., Bagnell, J.~A., Dey, A.~K., et~al.
\newblock Maximum entropy inverse reinforcement learning.
\newblock In \emph{AAAI}, volume~8, pp.\  1433--1438, 2008.

\end{thebibliography}
\bibliographystyle{icml2023}

\appendix
\onecolumn
\section{Proofs of Theorems}

\begin{proof}[Proof of Theorem \ref{theorem:grad-theta}]
The gradient of the training objective with respect to model parameters $\theta$ is:
\begin{align}
&\nabla_\theta \mathbb{E}_{\hat{\bf y}|\vec{X} \sim \hat{P}_\theta}\left[\sum_k \overbrace{\min_{\alpha_k} \left(
\operatorname{subdom}^k_{\alpha_k}\left(\hat{\bf y}, \tilde{\Yvec}, \yvec, \vec{a}\right)+\lambda_k\alpha_k\right)}^{\Gamma_k\left(\hat{\yvec}, \tilde{\Yvec}, \yvec, \vec{a}\right)}\right] 
=\mathbb{E}_{\hat{\bf y}|{\bf X} \sim \hat{P}_\theta}\bigg[\left(\sum_k \Gamma_k(\hat{\bf y}, \tilde{\Yvec}, \yvec, \vec{a})\right) \nabla_\theta\log \hat{\Pbb}_\theta(\hat{\bf y}|{\Xvec})\bigg], \notag
\end{align}
which follows directly from a property of gradients of logs of function:
\begin{align}
\nabla_\theta \log \hat{\mathbb{P}}(\hat{\yvec}|\Xvec) =
\frac{1}{\hat{\mathbb{P}}(\hat{\yvec}|\Xvec)} \nabla_\theta \hat{\mathbb{P}}(\hat{\yvec}|\Xvec) \implies
\nabla_{\theta} \hat{\mathbb{P}}_\theta(\hat{\yvec}|\Xvec) = \hat{\mathbb{P}}(\hat{\yvec}|\Xvec) \nabla_\theta \log \hat{\mathbb{P}}(\hat{\yvec}|\Xvec).
\end{align}
We note that this is a well-known approach employed by policy-gradient methods in reinforcement learning \cite{sutton2018reinforcement}.

Next, we consider how to obtain the $\alpha-$minimized subdominance for a particular tuple ($\hat{\yvec}$,$\tilde{\Yvec}$,$\yvec$,${\bf a}$),
$\Gamma_k\left(\hat{\yvec}, \tilde{\Yvec}, \yvec, \vec{a}\right)=
\min_{\alpha_k} \left(
\operatorname{subdom}^k_{\alpha_k}\left(\hat{\bf y}, \tilde{\Yvec}, \yvec, \vec{a}\right)+\lambda_k\alpha_k\right)$,
analytically.

First, we note that $\operatorname{subdom}^k_{\alpha_k}\left(\hat{\bf y}, \tilde{\Yvec}, \yvec, \vec{a}\right)+\lambda_k\alpha_k$ is comprised of hinged linear functions of $\alpha_k$, making it a convex and piece-wise linear function of $\alpha_k$.
This has two important implications: (1) any point of the function for which the subgradient includes $0$ is a global minimum of the function \cite{boyd2004convex}; (2) an optimum must exist at a corner of the function: $\alpha_k=0$ or where one of the hinge functions becomes active: 
\begin{align}
\alpha_k(f_k(\hat{\yvec}_i)-f_k(\tilde{\yvec}_i))+1 = 0 \implies
\alpha_k = \frac{1}{f_k(\tilde{\yvec}_i)-f_k(\hat{\yvec}_i)}.
\end{align}
The subgradient for the $j^{\text{th}}$ of these points (ordered by $f_k$ value from smallest to largest and denoted $f_k(\tilde{\yvec}^{(j)})$ for the demonstration) is:
\begin{align}
& \partial_{\alpha_k}
\operatorname{subdom}^k_{\alpha_k}\left(\hat{\bf y}, \tilde{\Yvec}, \yvec, \vec{a}\right)\Big|_{\alpha_k=(f_k(\hat{\yvec})-f_k(\tilde{\yvec}^{(j)}))^{-1}} =\partial_{\alpha_k}\left(\frac{1}{N}\sum_{i=1}^{j}\bigg[\alpha_k\bigg(f_k(\hat{\yvec})-f_k(\tilde{\yvec}^{(i)})\bigg)+1\bigg]_{+}+\lambda\alpha_k \right)
\notag\\
& = \lambda + \frac{1}{N}\sum_{i=1}^{j-1} \bigg(f_k(\hat{\yvec})-f_k(\tilde{\yvec}^{(i)})\bigg) + \left[0, f_k(\hat{\yvec})-f_k(\tilde{\yvec}^{(j)})\right], \notag
\end{align}
where the final bracketed expression indicates the range of values added to the constant value preceding it.

The smallest $j$ for which the largest value in this range is positive must contain the $0$ in its corresponding range, and is thus the provides the $j$ value for the optimal $\alpha_k$ value. 
\end{proof}

\begin{proof}[Proof of Theorem \ref{thm:generalization}]
We extend the leave-one-out generalization bound of \citet{ziebart2022towards} by considering the set of reference decisions that are support vectors for any learner decisions with non-zero probability.
For the remaining reference decisions that are not part of this set, removing them from the training set would not change the optimal  model choice and thus contribute zero error to the leave-one-out cross validation error, which is an almost unbiased estimate of the generalization error \cite{vapnik2000bounds}.

\end{proof}

\section{Additional Results}
In the main paper, we only included plots that show the relationship of a fairness metric with \emph{prediction error}. To show the relation between each pair of fairness metrics, in Figures \ref{fig:result3} and \ref{fig:result4} we show the remaining plots removed from Figures \ref{fig:result1} and \ref{fig:result2} respectively.

\begin{figure*}[t]
    \includegraphics[width=.34\textwidth]{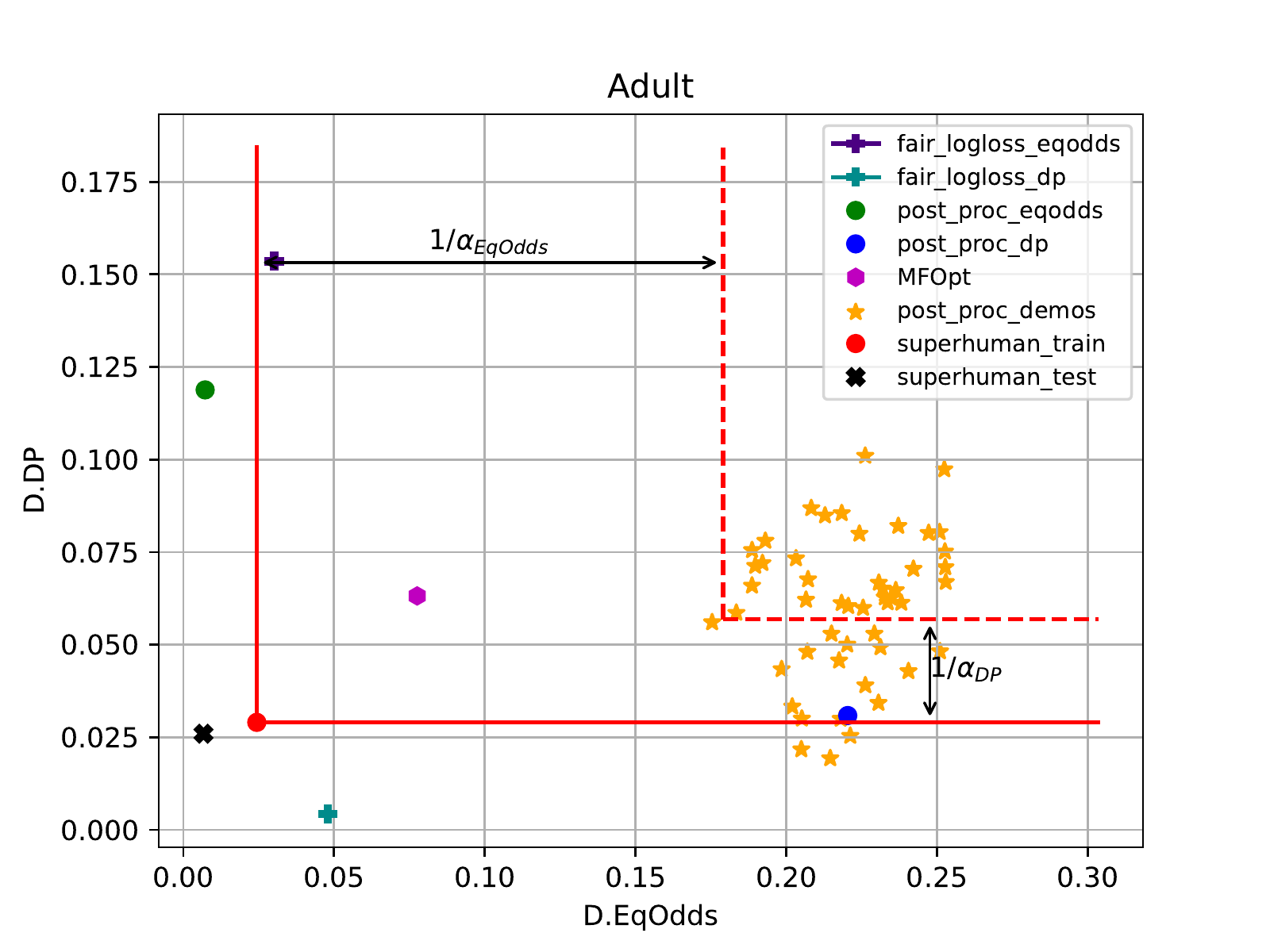} 
    \includegraphics[width=.34\textwidth]{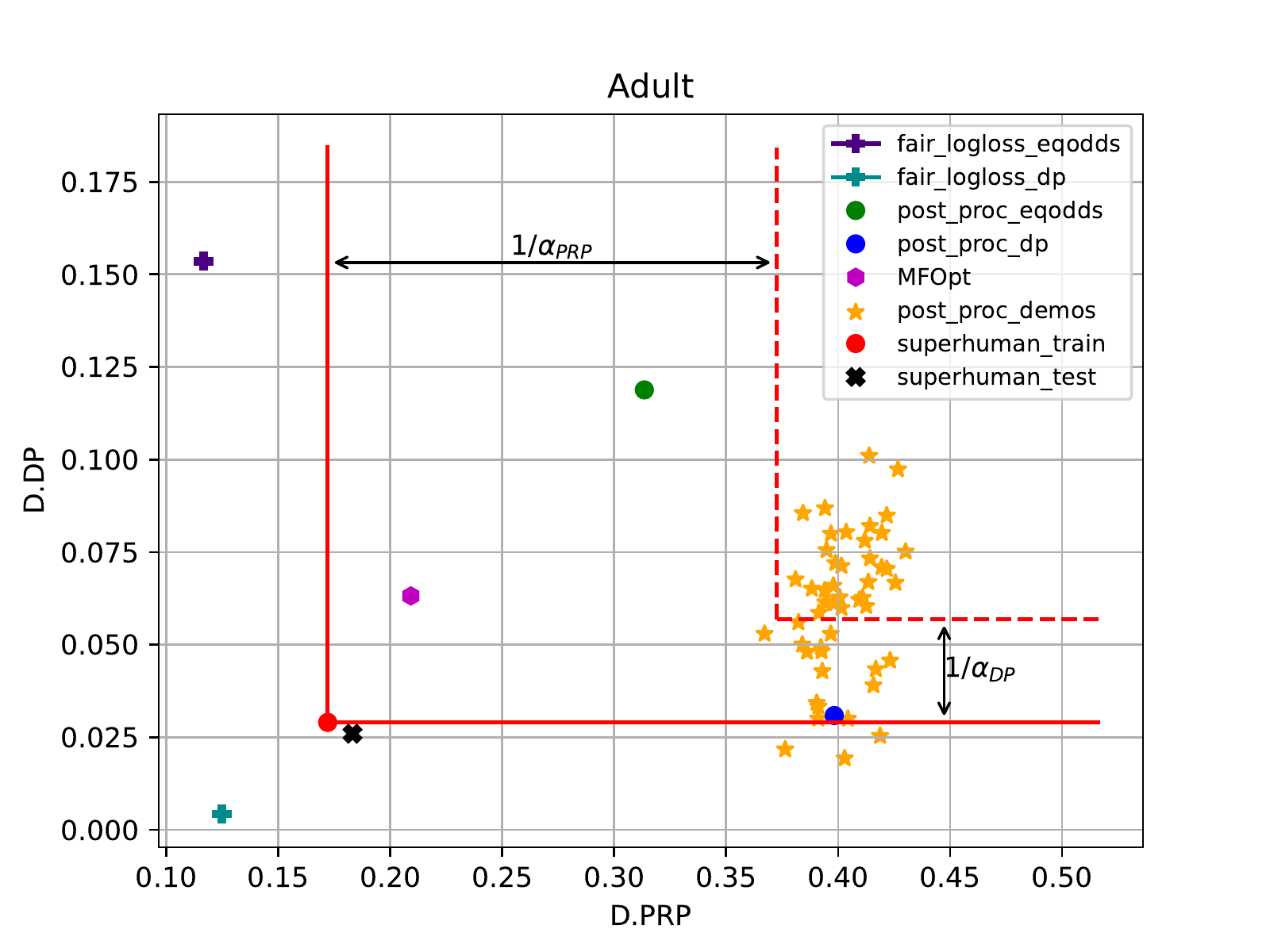}
    \includegraphics[width=.34\textwidth]{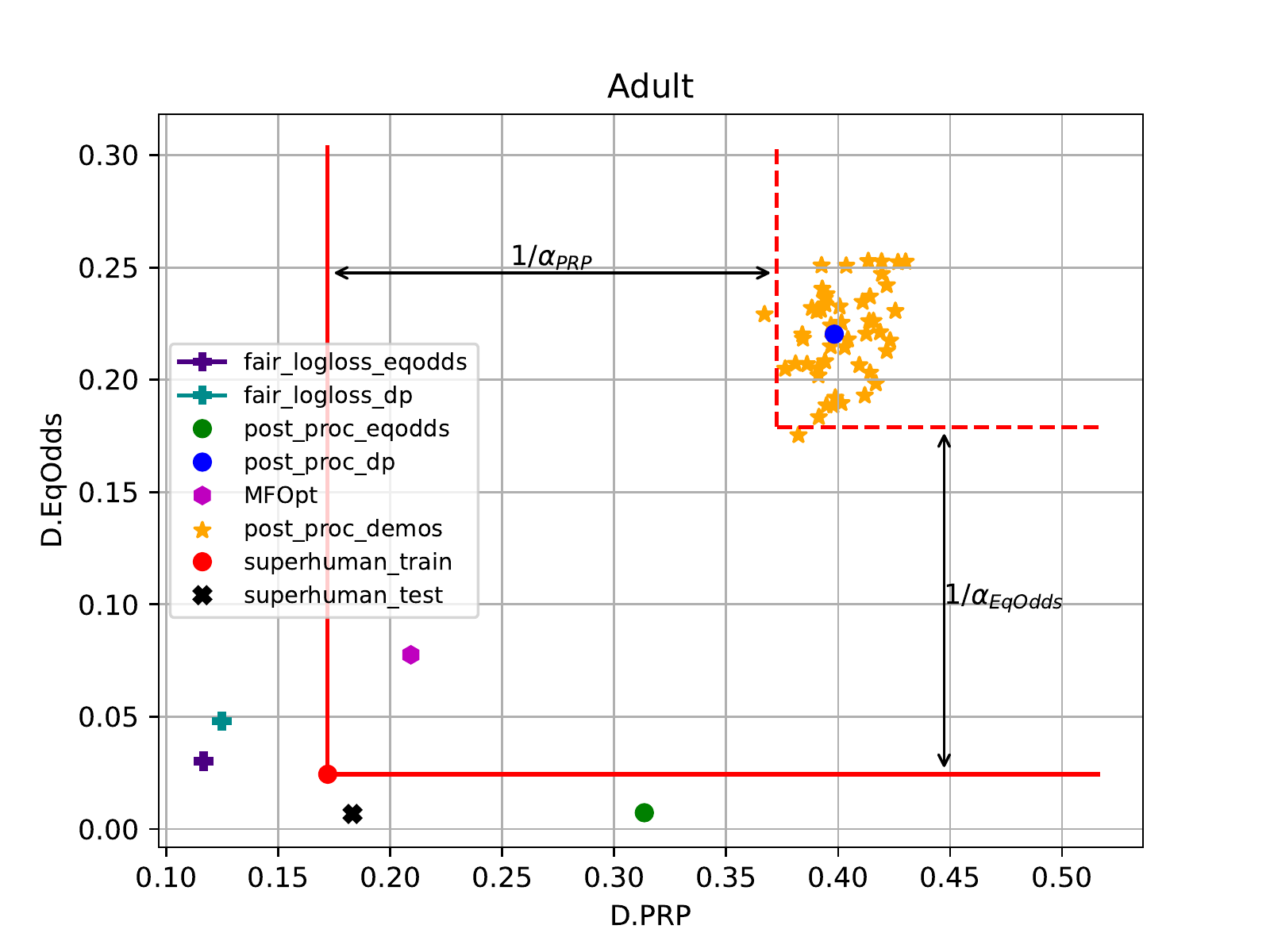}\\
    \includegraphics[width=.34\textwidth]{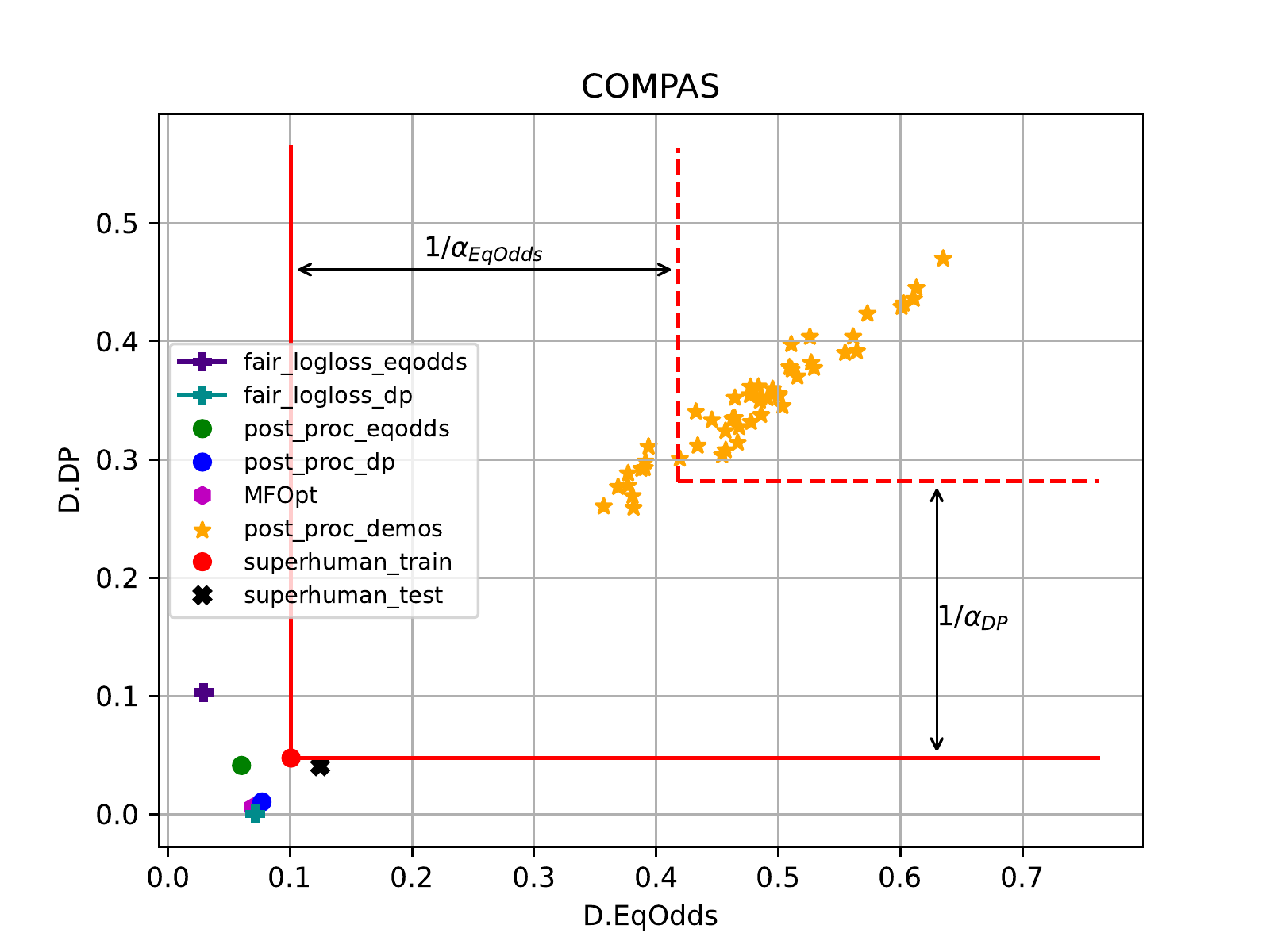} 
    \includegraphics[width=.34\textwidth]{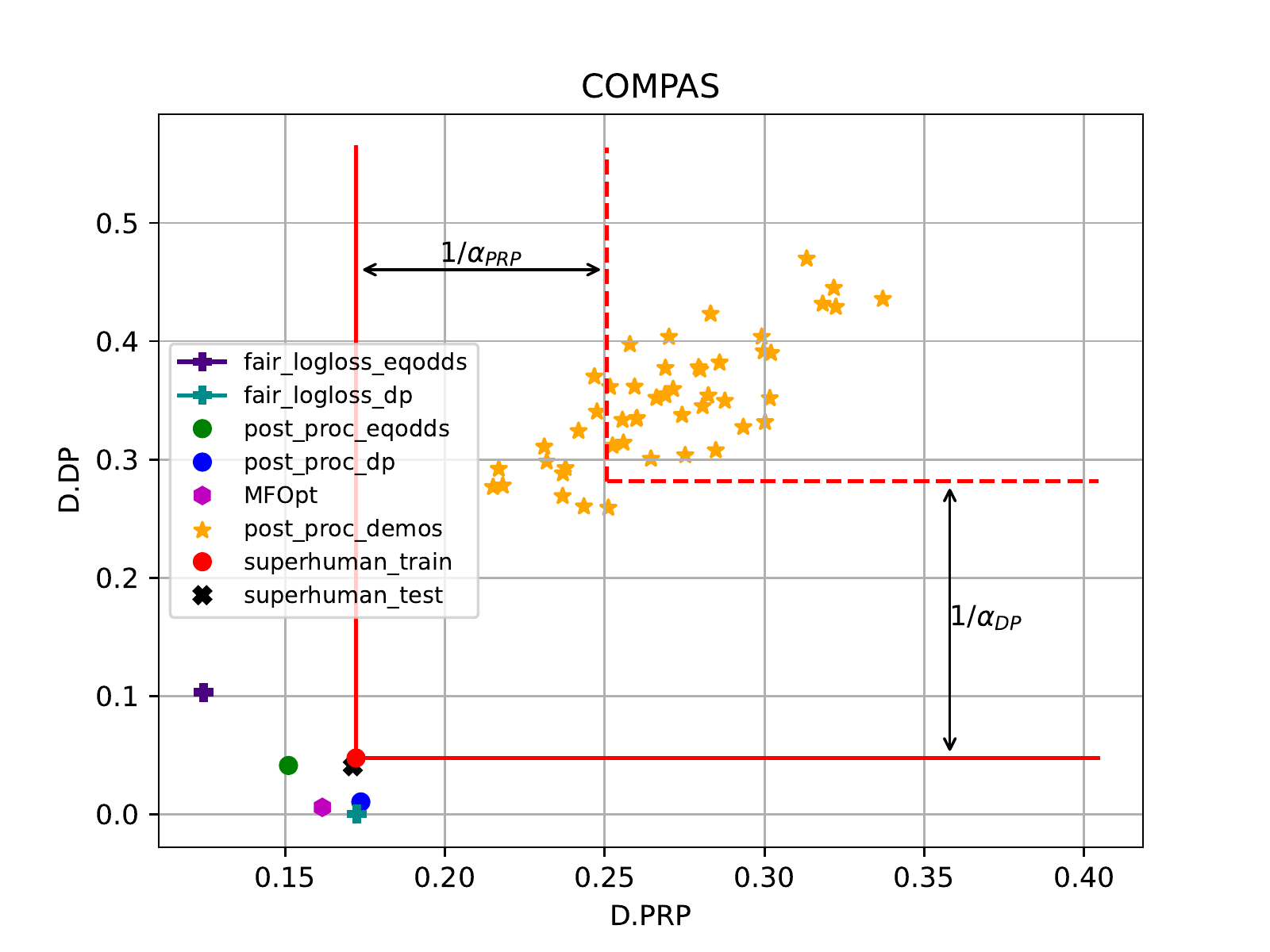}
    \includegraphics[width=.34\textwidth]{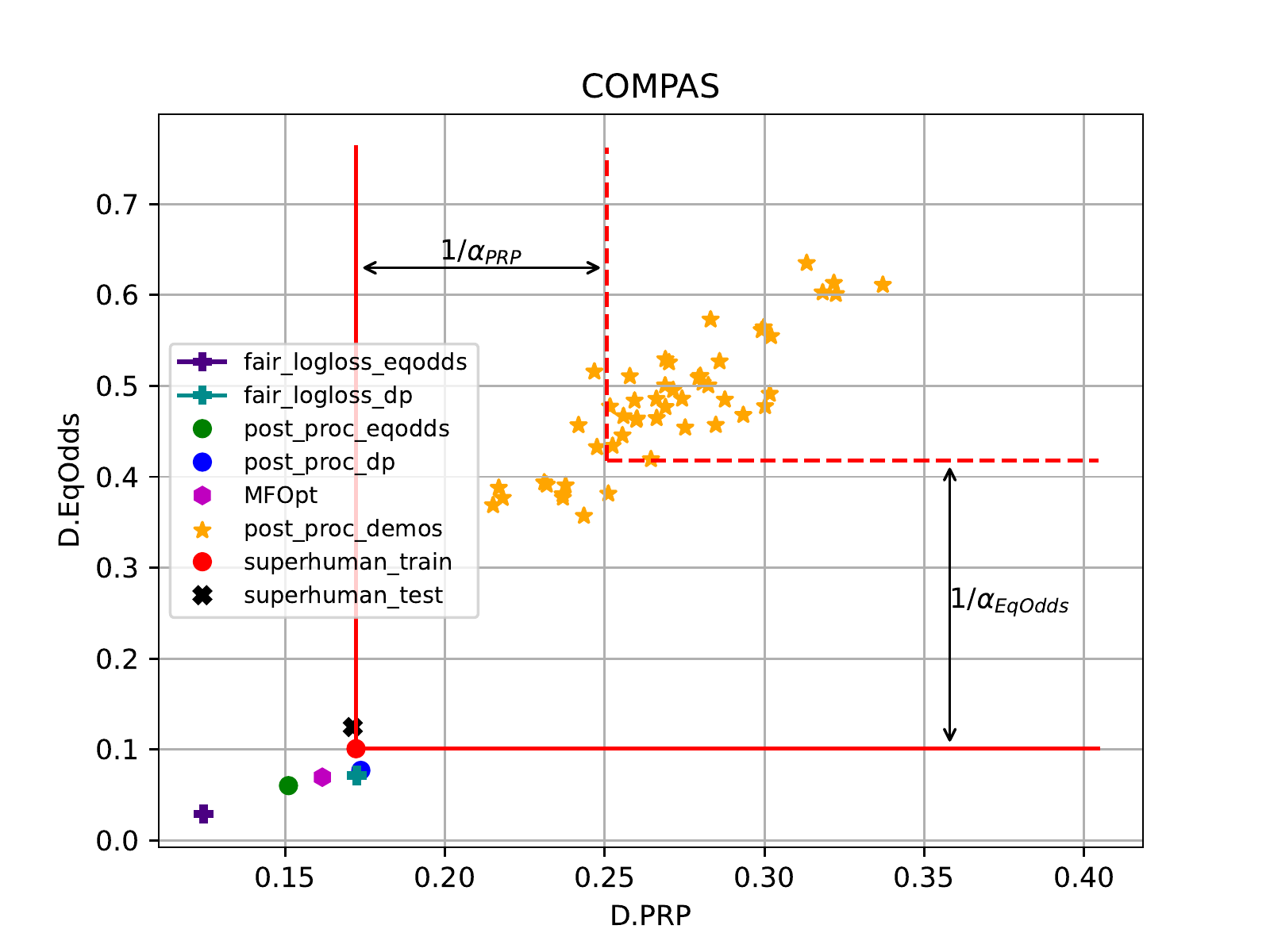}

    \caption{\small
    The trade-off between each pair of: {\it difference of} \emph{Demographic Parity} ({\texttt{D.DP}}), \emph{Equalized Odds} ({\texttt{D.EqOdds}}) and \emph{Predictive Rate Parity} ({\texttt{D.PR}}) on test data using noiseless training data ($\epsilon=0$) for \texttt{Adult} (top row) and \texttt{COMPAS} (bottom row) datasets.}
    \label{fig:result3}
\end{figure*}

\begin{figure*}[t]
    \includegraphics[width=.33\textwidth]{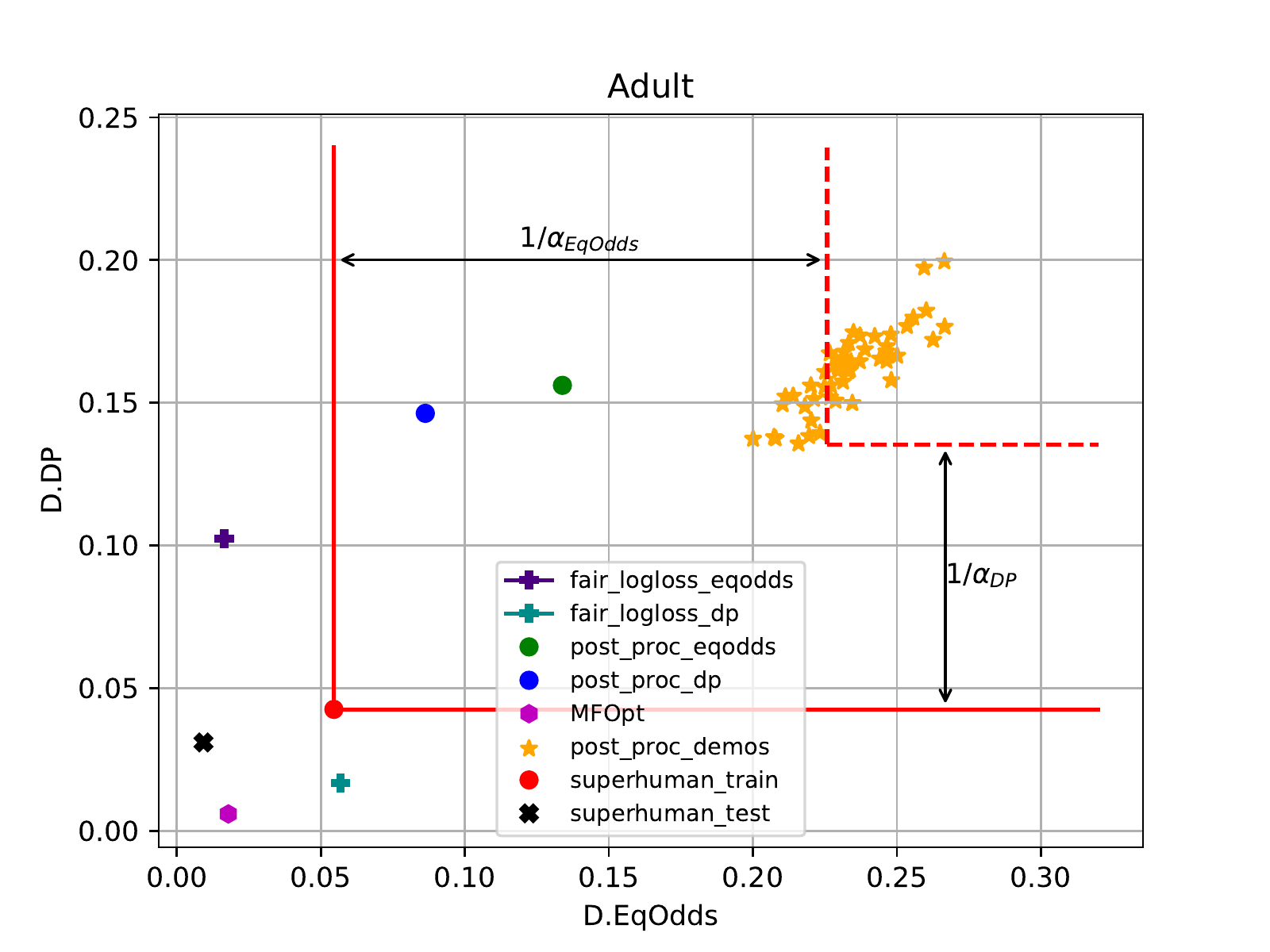} 
    \includegraphics[width=.33\textwidth]{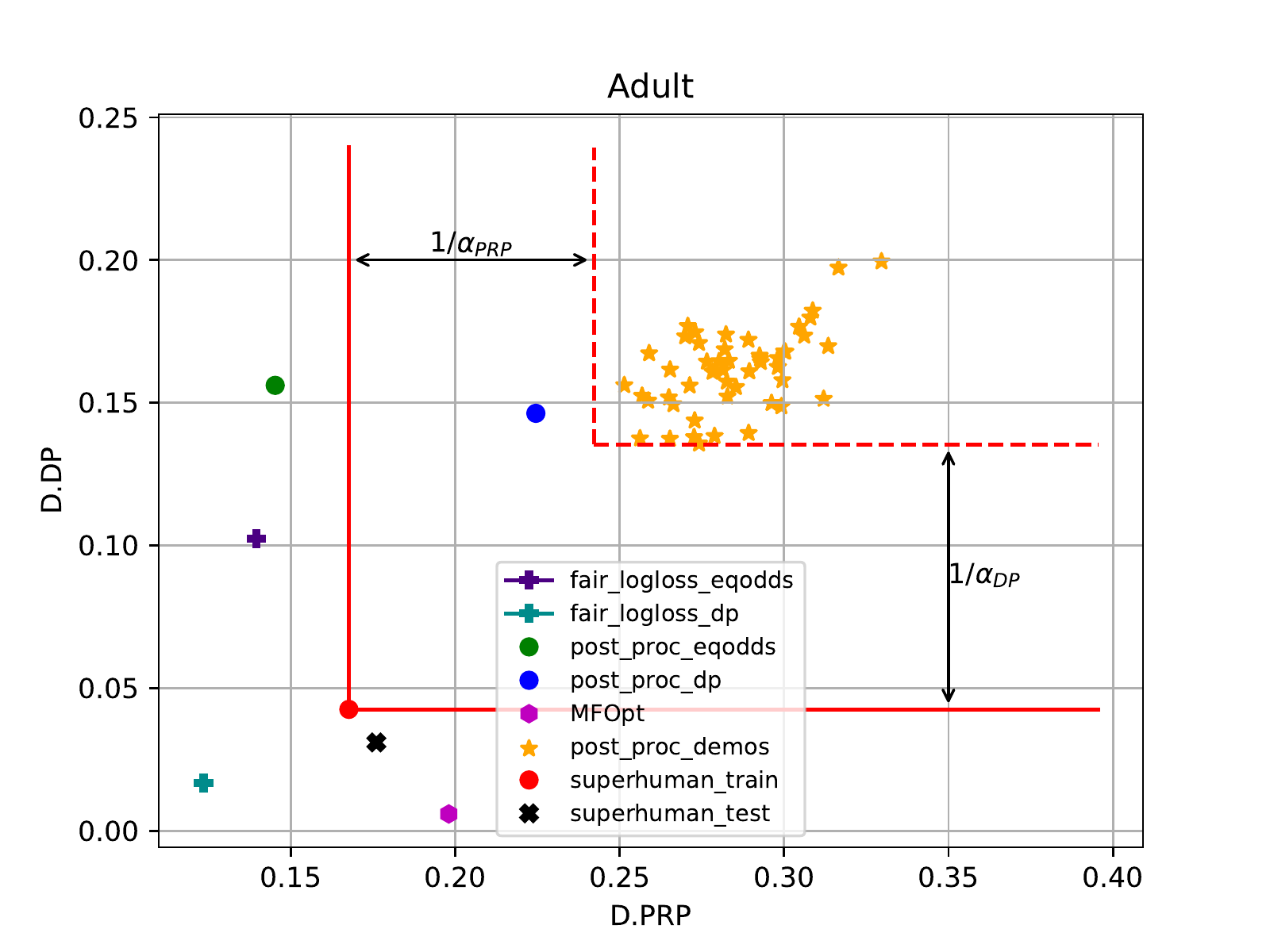}
    \includegraphics[width=.33\textwidth]{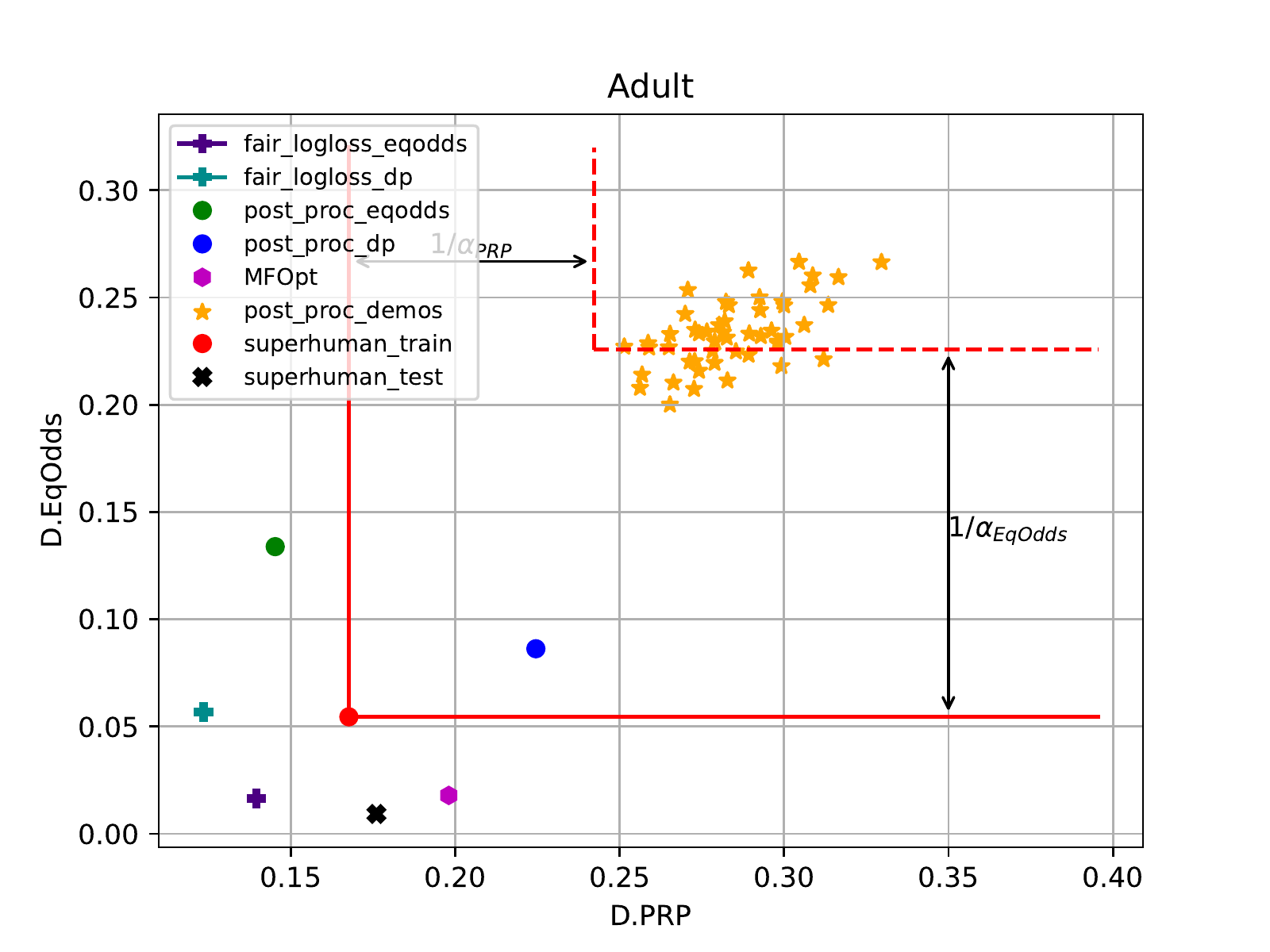}\\
    \includegraphics[width=.33\textwidth]{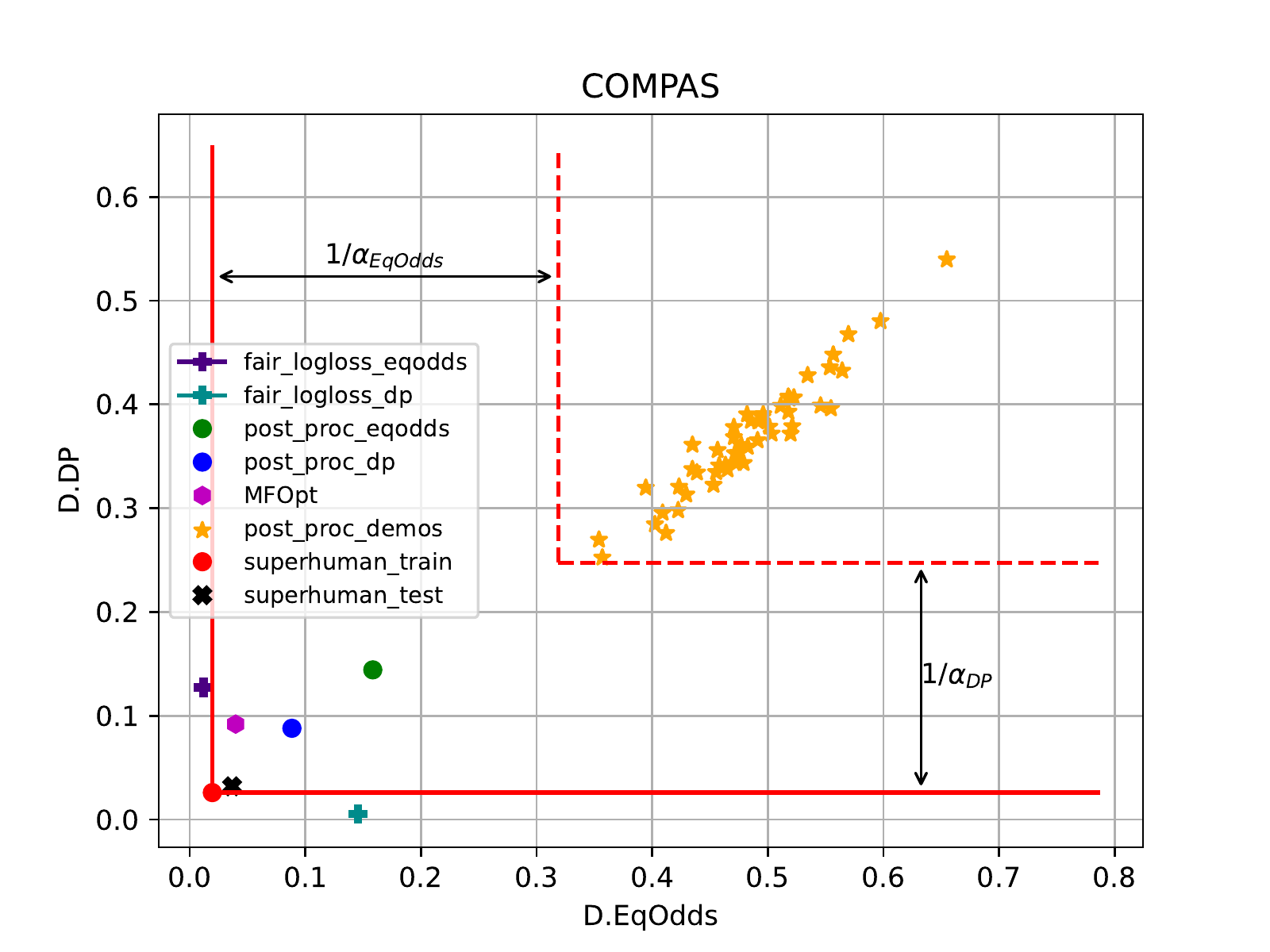} 
    \includegraphics[width=.33\textwidth]{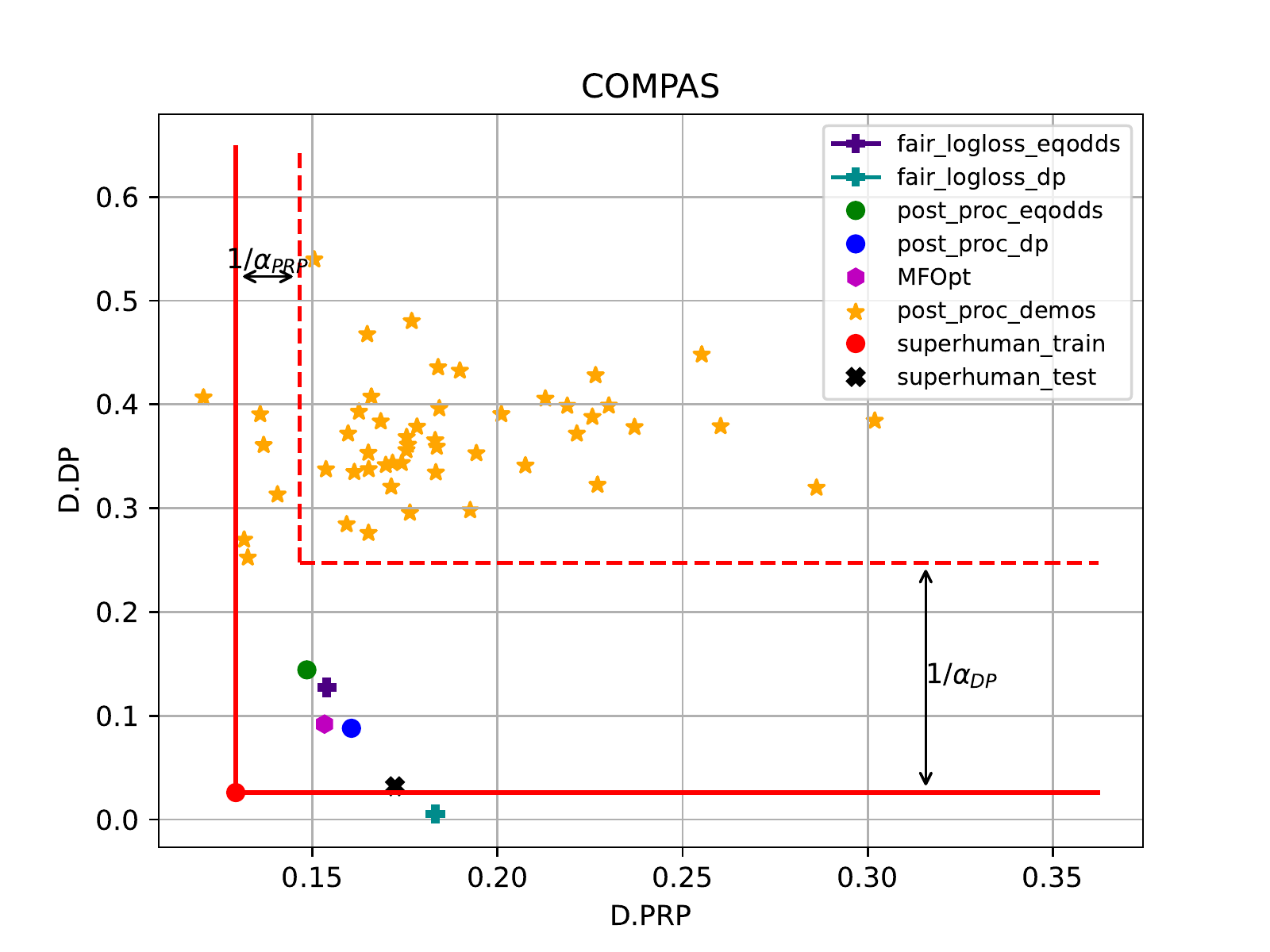}
    \includegraphics[width=.33\textwidth]{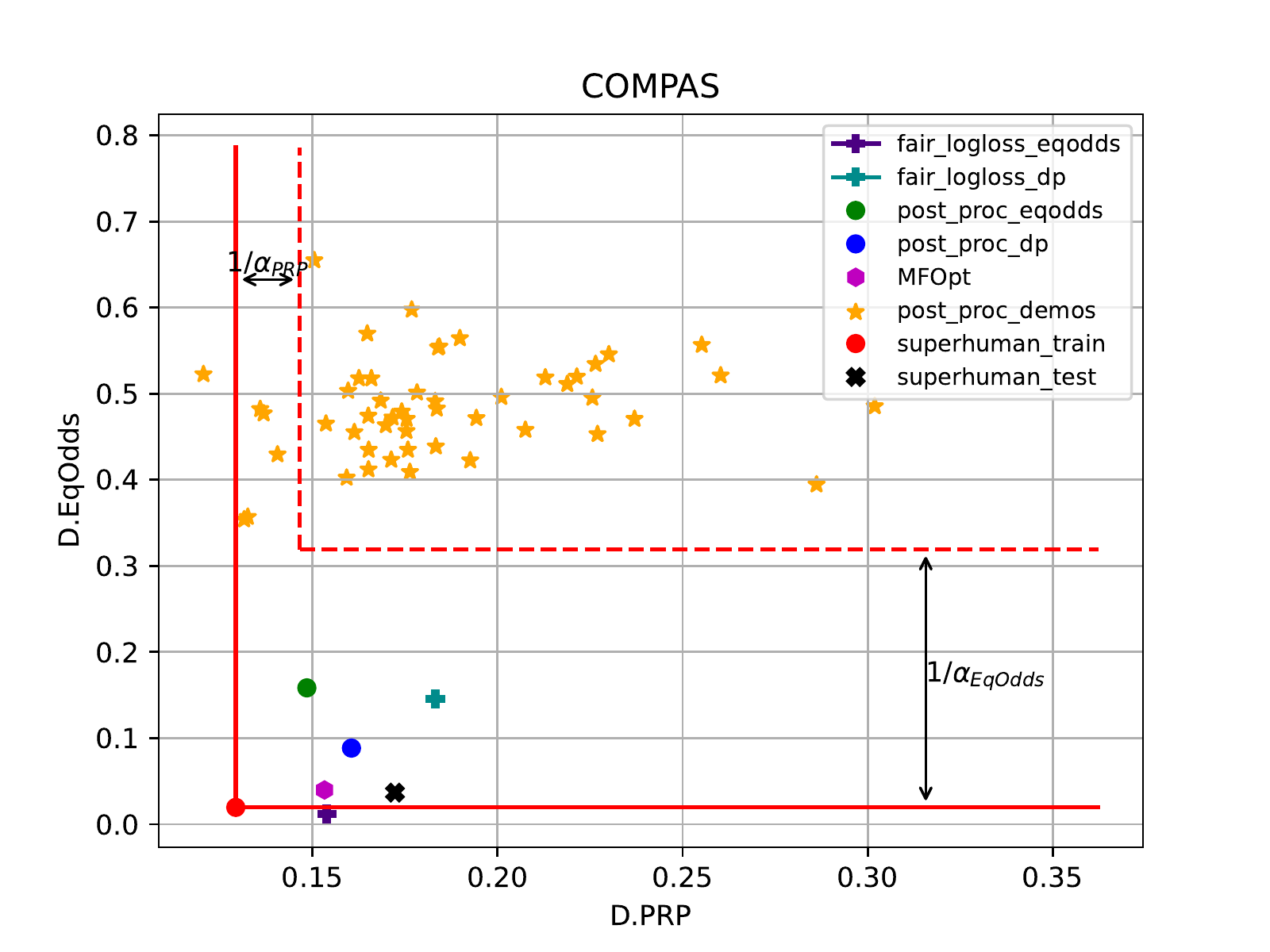}
    \caption{\small
    The trade-off between each pair of: {\it difference of} \emph{Demographic Parity} ({\texttt{D.DP}}), \emph{Equalized Odds} ({\texttt{D.EqOdds}}) and \emph{Predictive Rate Parity} ({\texttt{D.PR}}) on test data using noiseless training data ($\epsilon=0.2$) for \texttt{Adult} (top row) and \texttt{COMPAS} (bottom row) datasets.}
    \label{fig:result4}
\end{figure*}

\begin{figure*}[t]
    \includegraphics[width=.34\textwidth]{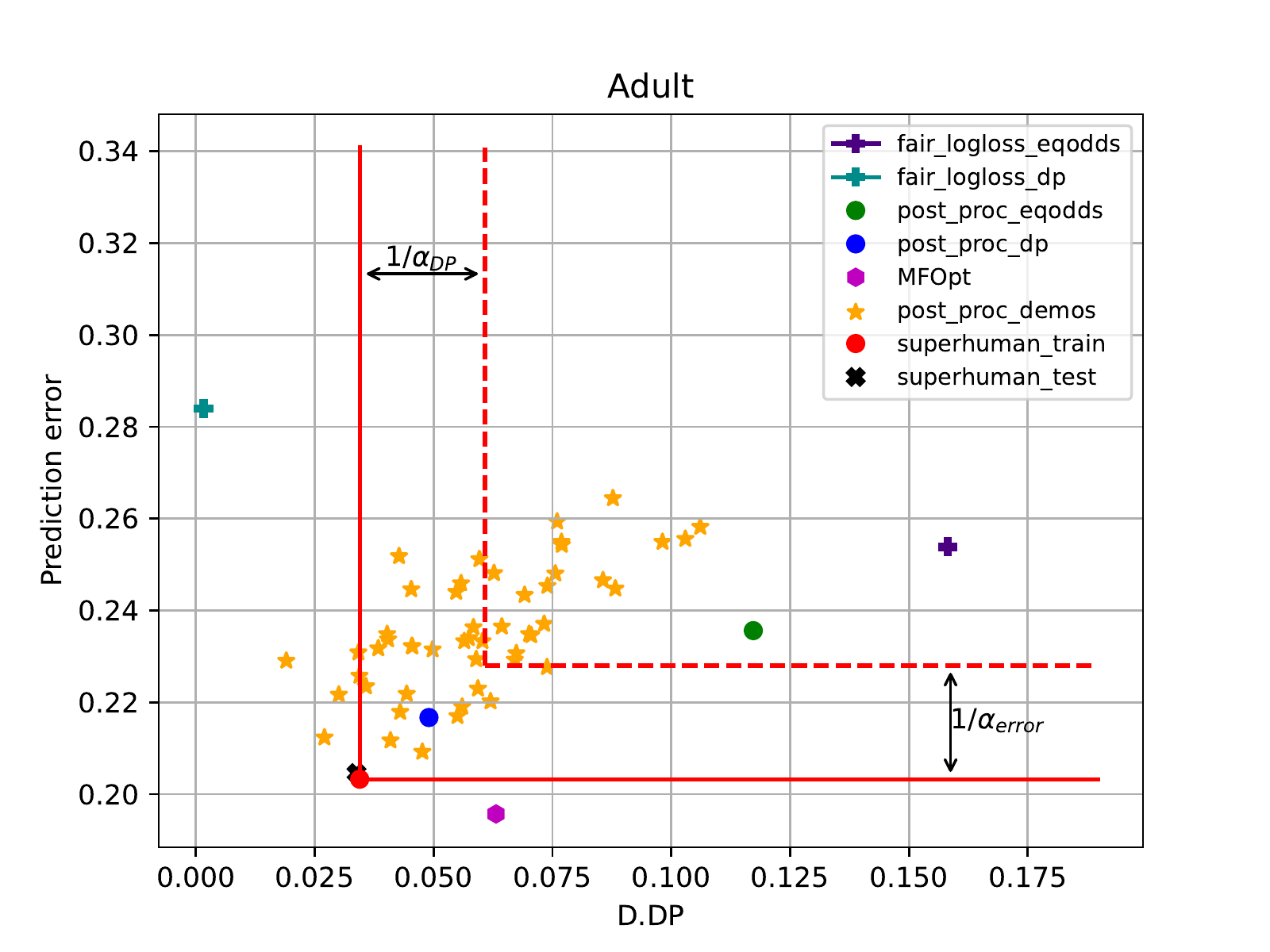} 
    \includegraphics[width=.34\textwidth]{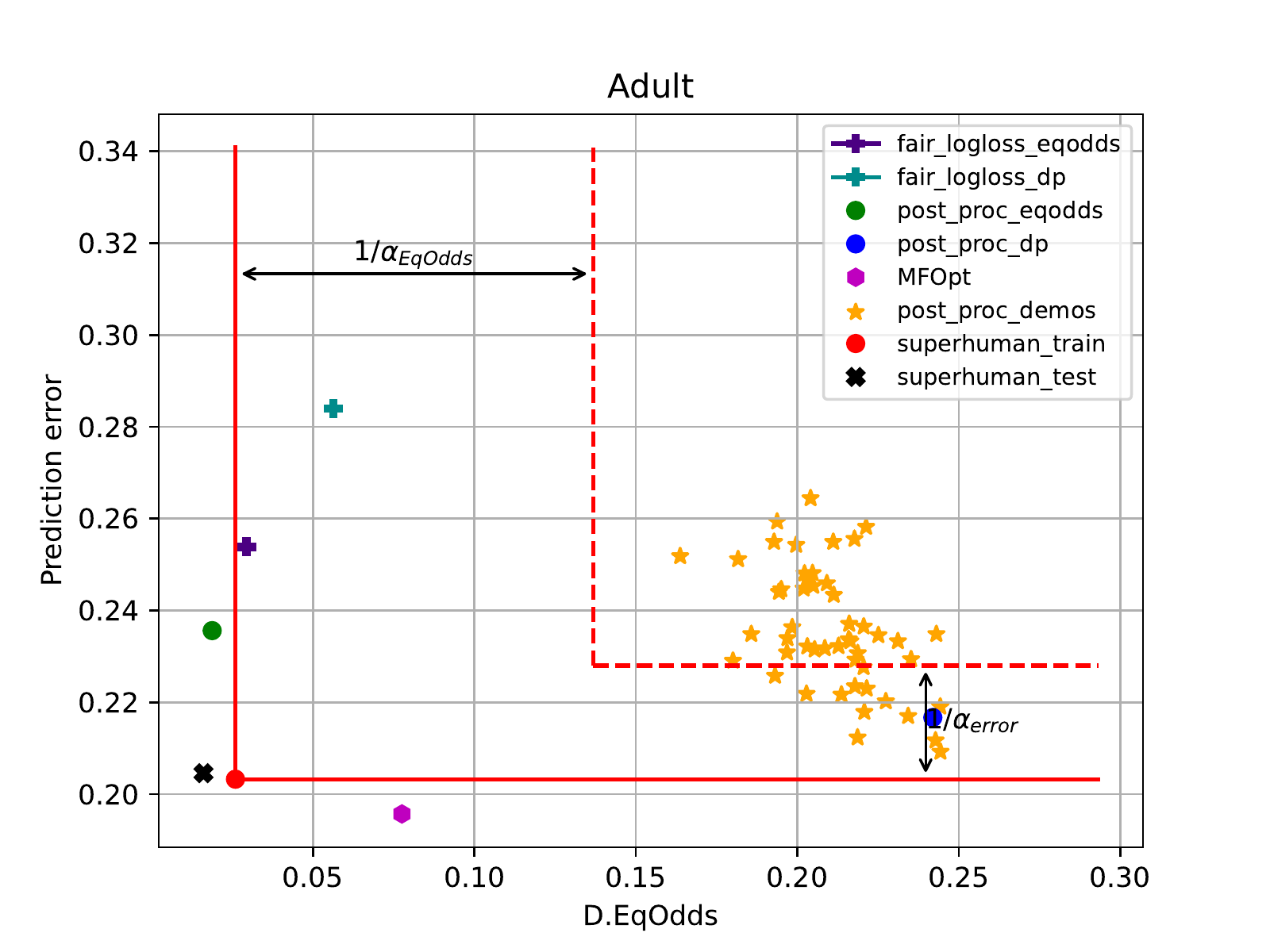}
    \includegraphics[width=.34\textwidth]{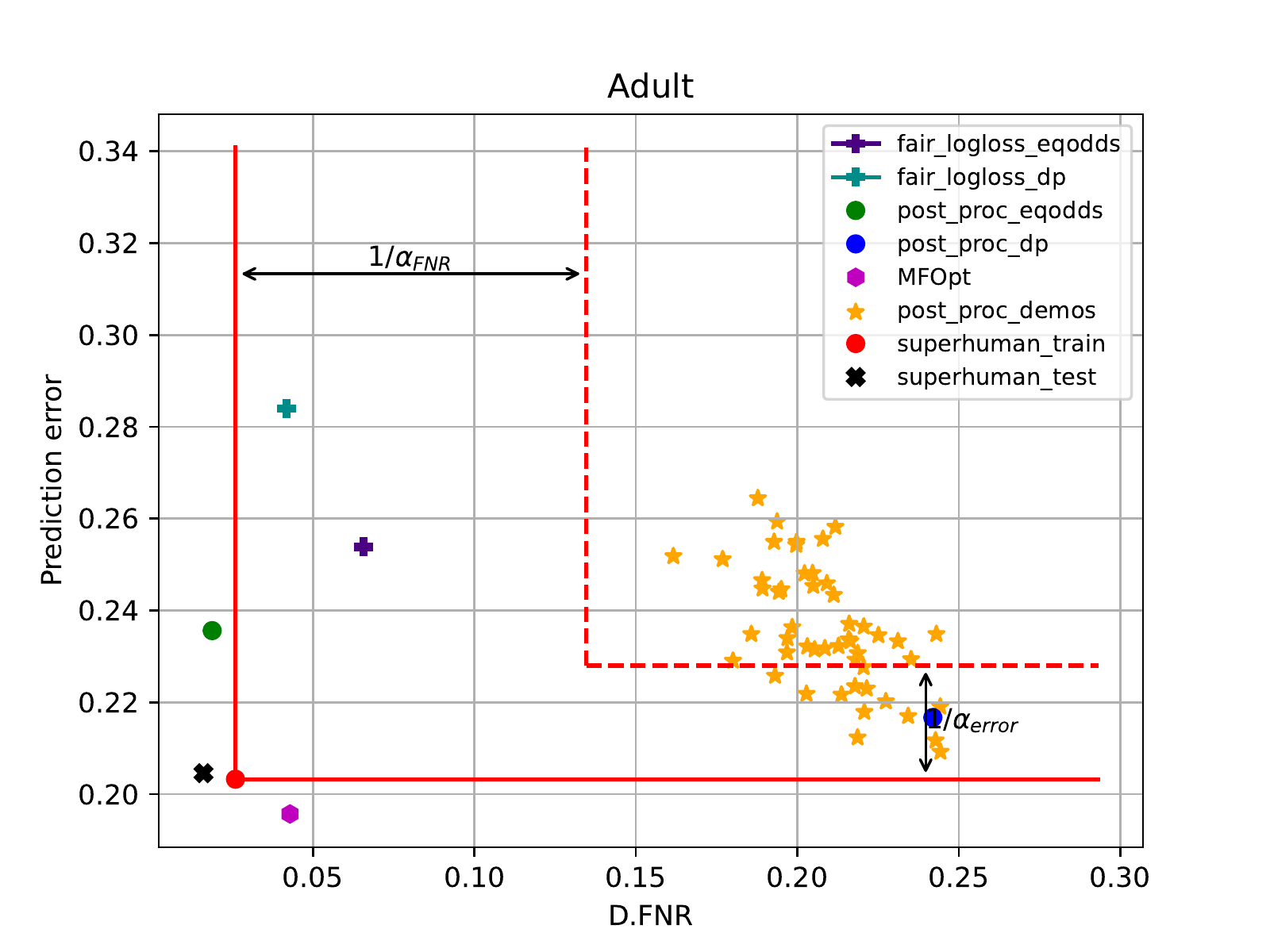}\\
    \includegraphics[width=.34\textwidth]{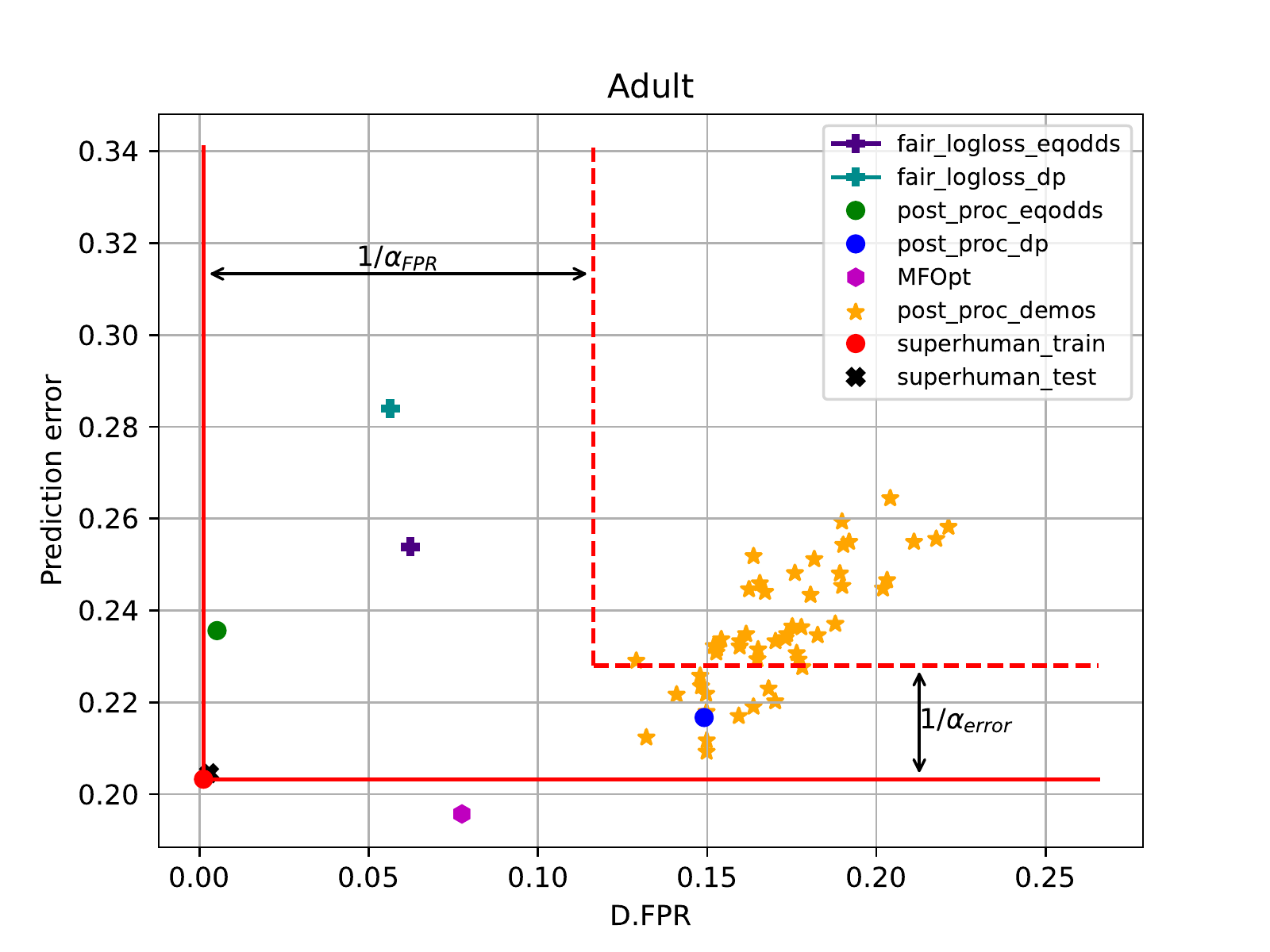} 
    \includegraphics[width=.34\textwidth]{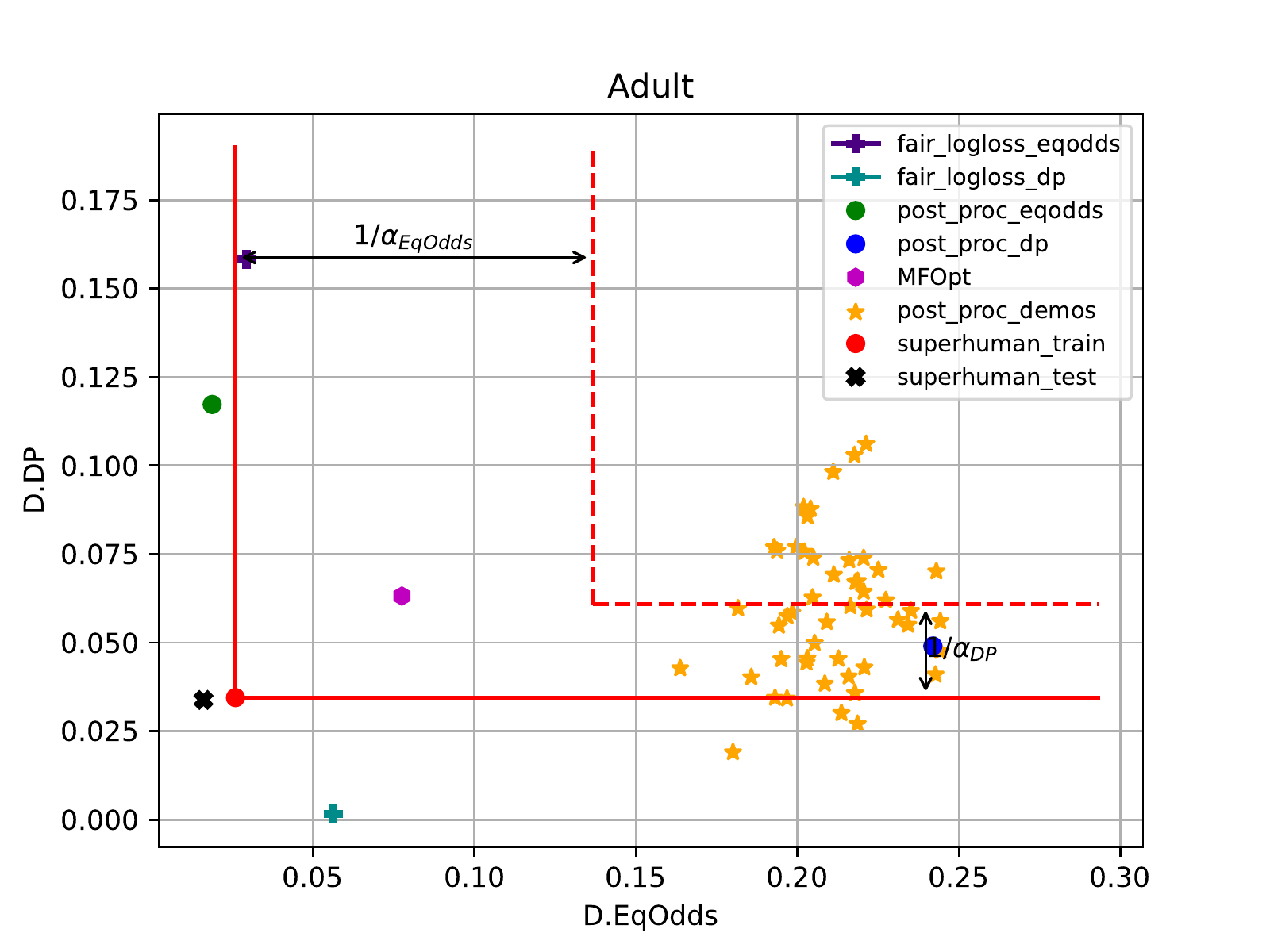}
    \includegraphics[width=.34\textwidth]{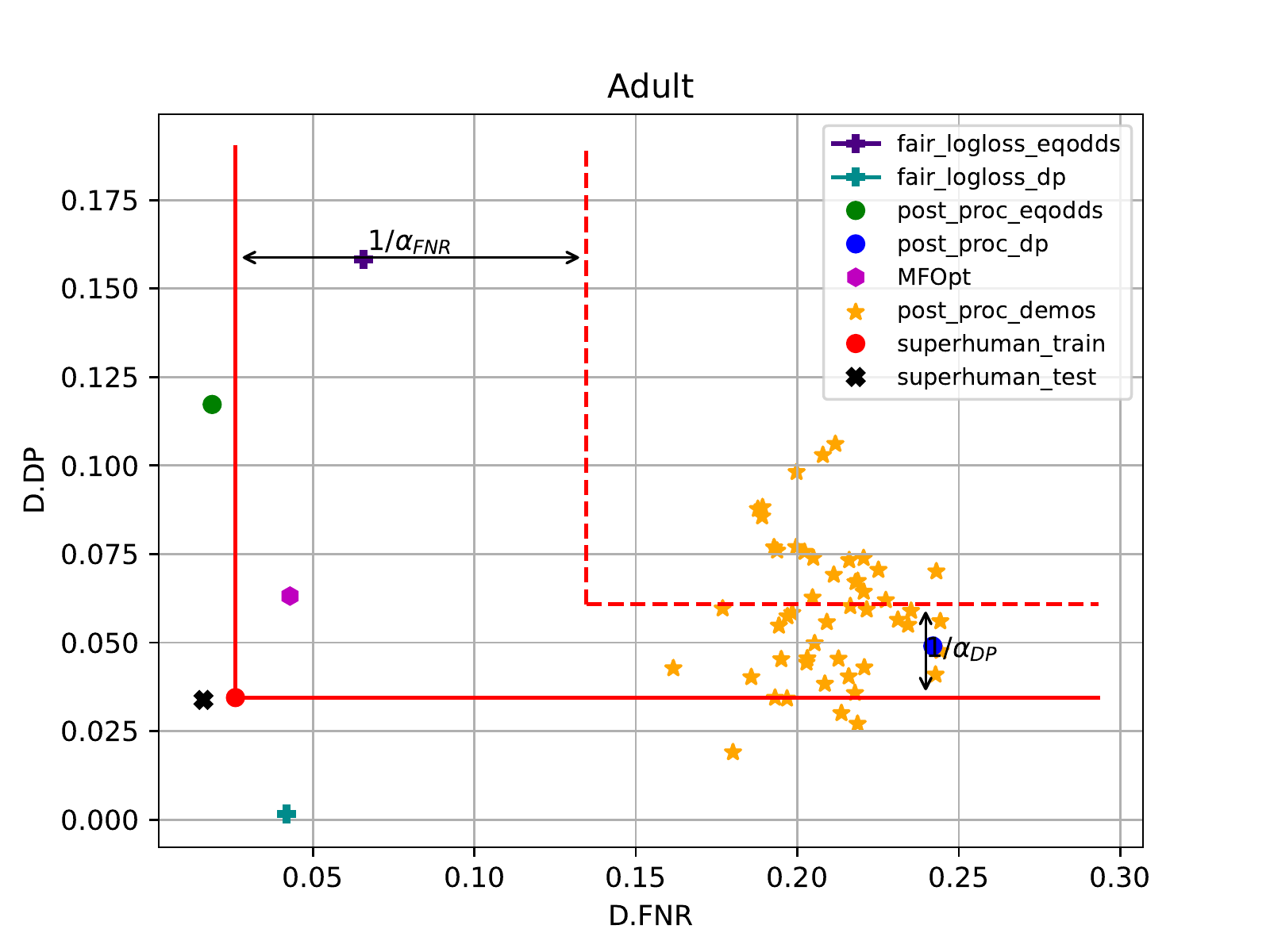}\\
    \includegraphics[width=.34\textwidth]{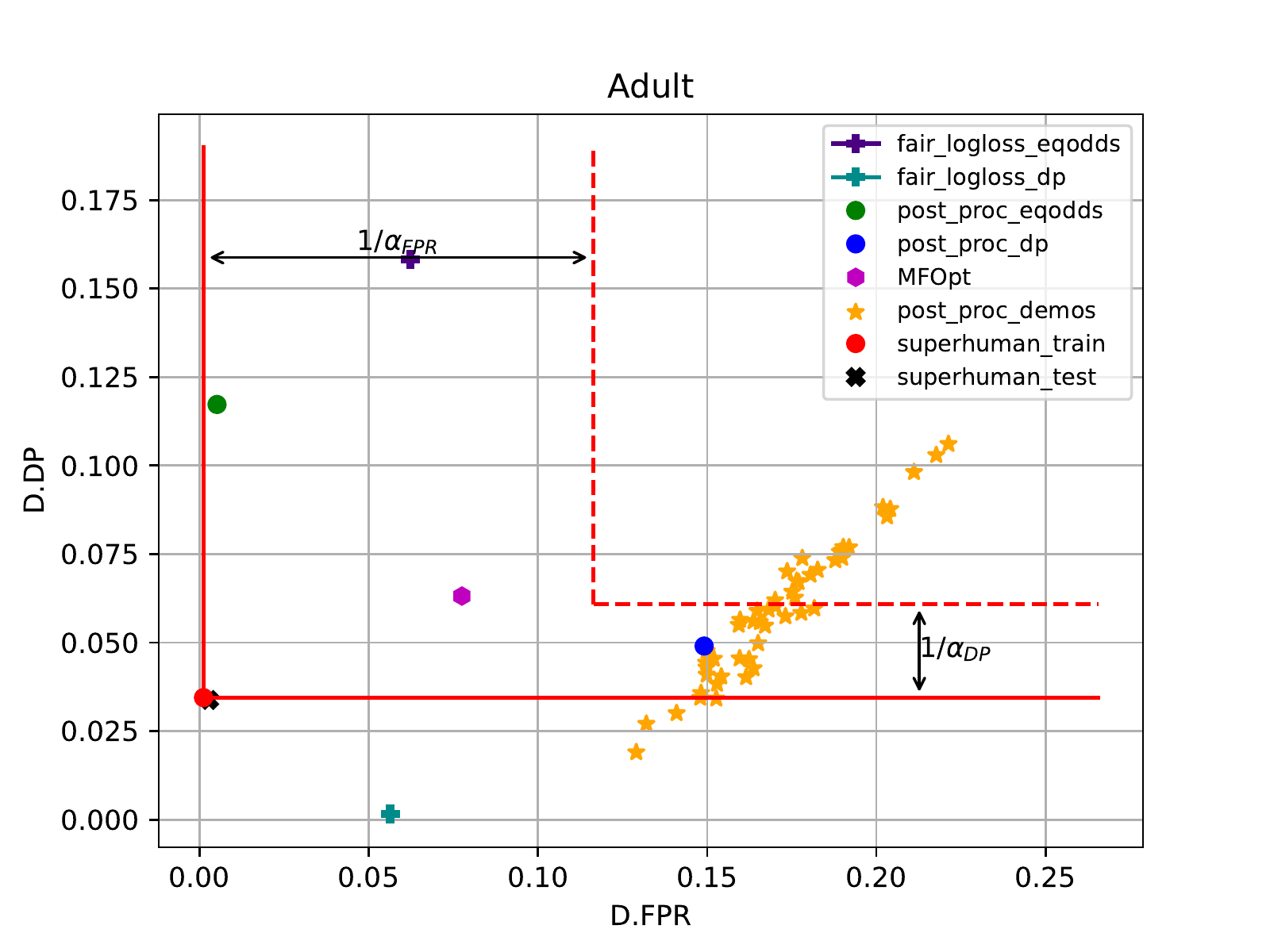} 
    \includegraphics[width=.34\textwidth]{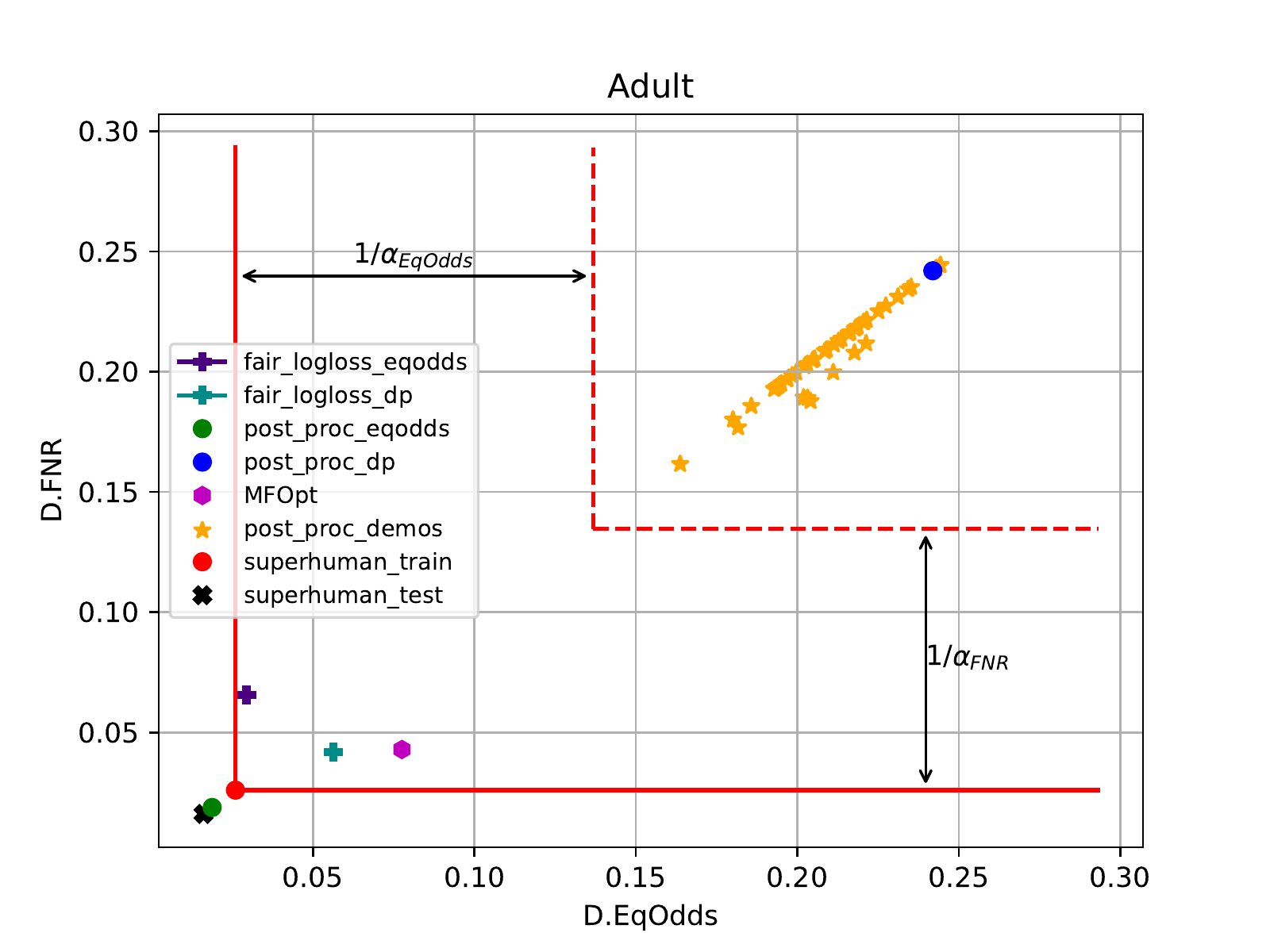}
    \includegraphics[width=.34\textwidth]{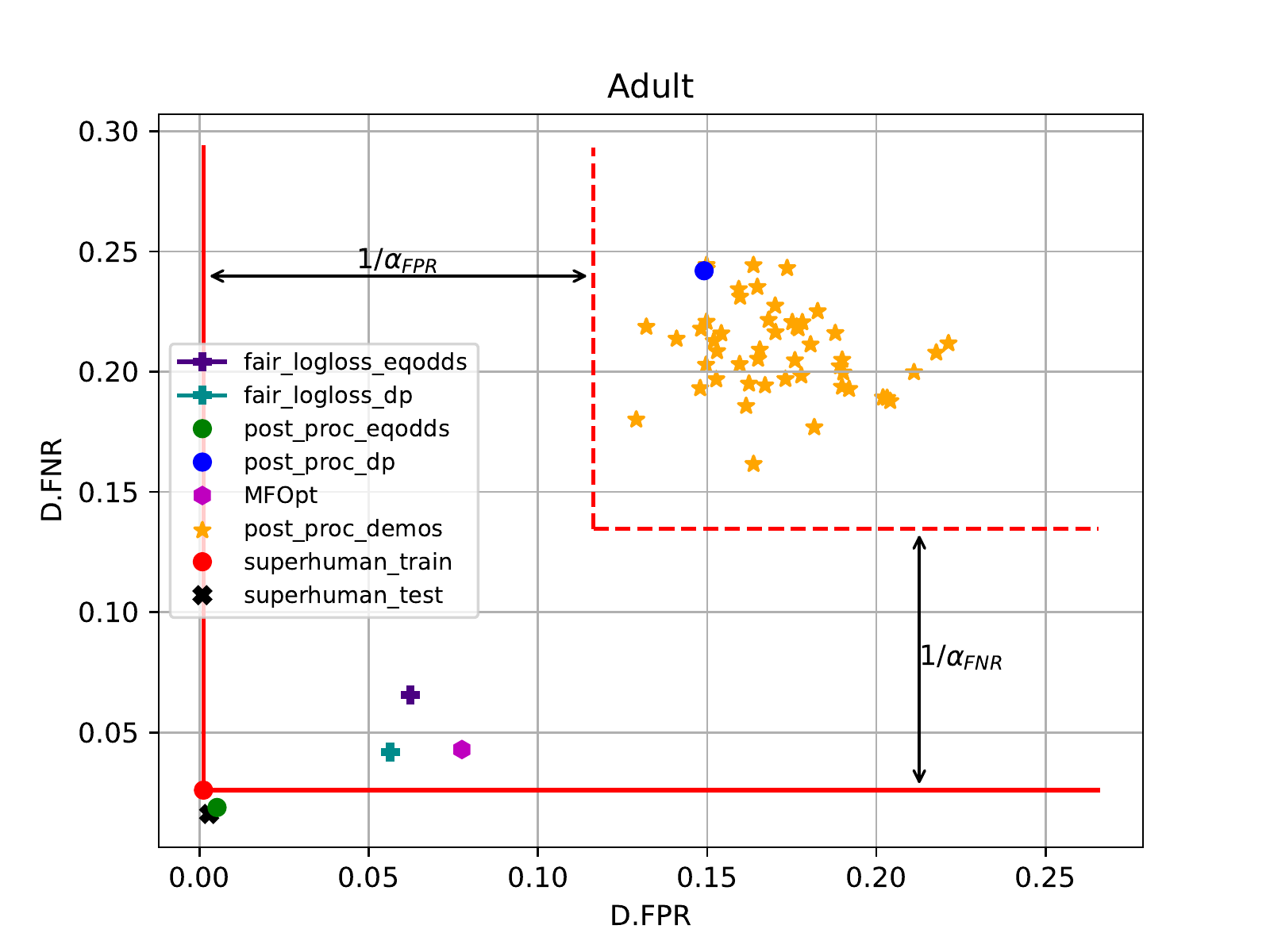}
    \\
    \vspace{-6mm}
    \begin{center}
    \quad
    \includegraphics[width=.34\textwidth]{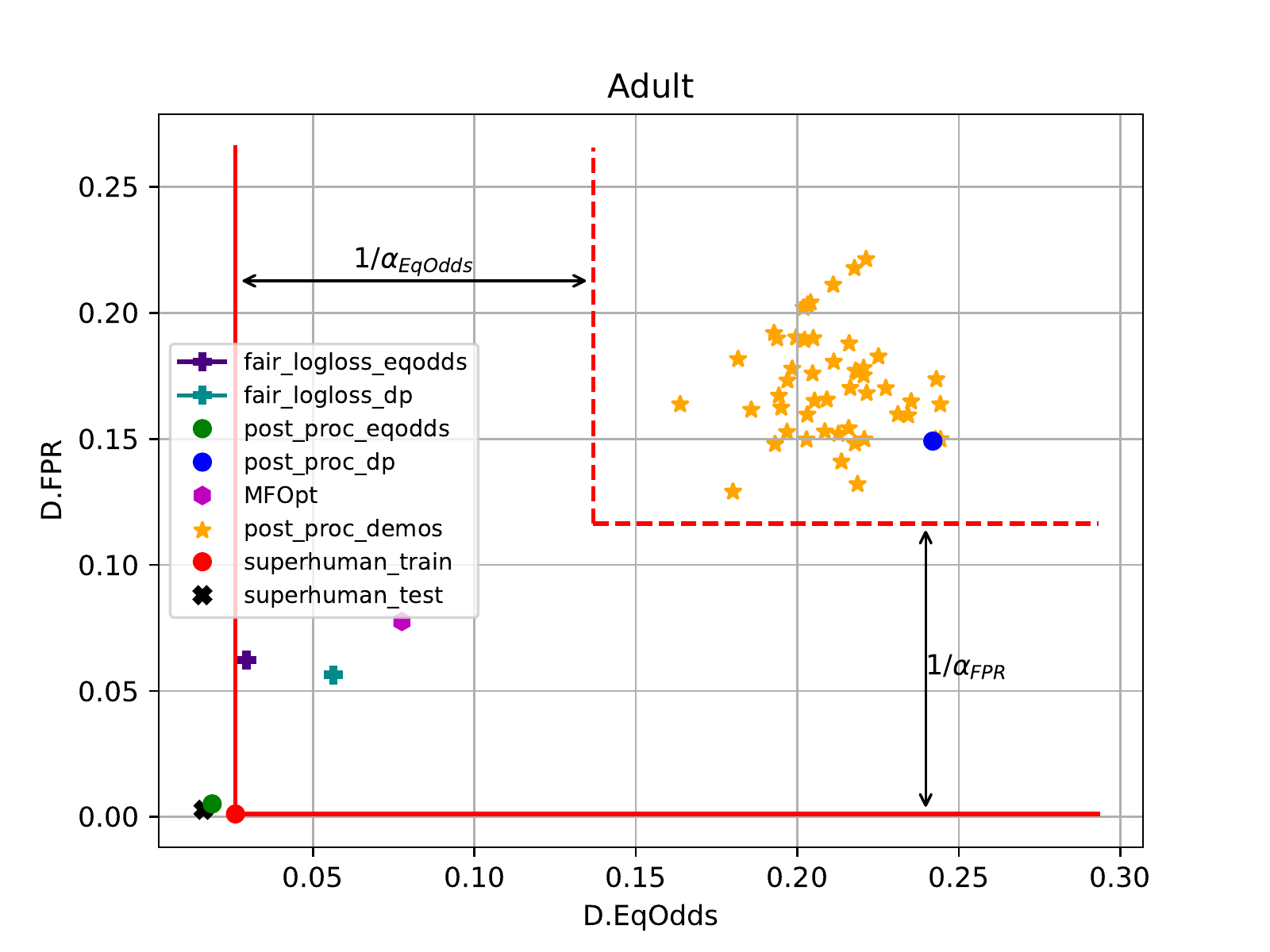}
    \end{center}

    \caption{\small
    The trade-off between each pair of: {\it difference of} \emph{Demographic Parity} ({\texttt{D.DP}}), \emph{Equalized Odds} ({\texttt{D.EqOdds}}), \emph{False Negative Rate} ({\texttt{D.FNR}}), \emph{False Positive Rate} ({\texttt{D.FPR}}) and \emph{Prediction Error} on test data using noiseless training data ($\epsilon=0$) for \texttt{Adult} dataset.}
    \label{fig:result5}
\end{figure*}

\subsection{Experiment with more measures}
Since our approach is flexible enough to accept wide range of fairness/performance measures, we extend the experiment on \texttt{Adult} to $K=5$ features. In this experiment we use \emph{Demographic Parity} ({\texttt{D.DP}}), \emph{Equalized Odds} ({\texttt{D.EqOdds}}), \emph{False Negative Rate} ({\texttt{D.FNR}}), \emph{False Positive Rate} ({\texttt{D.FPR}}) and \emph{Prediction Error} as the features to outperform reference decisions on. The results are shown in Figure \ref{fig:result5}.

\end{document}